\documentclass{article}

\usepackage{microtype}
\usepackage{graphicx}
\usepackage{xcolor}
\usepackage{subfigure}
\usepackage{booktabs} 
\usepackage{algorithm}
\usepackage{algorithmicx}
\usepackage{hyperref}
\usepackage{float}

\usepackage{tikz}
\usetikzlibrary{fit}
\newcommand{\eqnode}[2]{%
  \tikz[remember picture, baseline=(#1.base)]{\node[inner sep=0pt, name=#1]{$\displaystyle #2$};}%
}

\usepackage{caption}

\definecolor{mydarkblue}{rgb}{0,0.08,0.45}
\definecolor{dred}{RGB}{153,80,43}

\definecolor{myblue}{HTML}{2a4d7f}
\makeatletter
\@namedef{ver@algorithmic.sty}{}
\makeatother
\usepackage[accepted]{icml2026}

\usepackage{amsmath}
\usepackage{amssymb}
\usepackage{mathtools}
\usepackage{amsthm}
\usepackage{algorithm}
\usepackage{algpseudocode}

\newcommand{\ablation}[1]{{\color{dred}{#1}}}
\newcommand{\dataset}[1]{\textsc{#1}}
\DeclareRobustCommand{\crblue}[1]{{\color{black}#1}}

\usepackage[capitalize,noabbrev]{cleveref}

\theoremstyle{plain}
\newtheorem{theorem}{Theorem}[section]

\theoremstyle{definition}

\newtheorem{assumption}[theorem]{Assumption}
\theoremstyle{remark}
\newtheorem{remark}[theorem]{Remark}

\usepackage[textsize=tiny]{todonotes}

\icmltitlerunning{Mesh Based Simulations with Spatial and Temporal awareness}

\begin{document}

\twocolumn[
\icmltitle{Mesh Based Simulations with Spatial and Temporal awareness}



\icmlsetsymbol{equal}{*}

\begin{icmlauthorlist}
\icmlauthor{Paul Garnier}{cemef}
\icmlauthor{Vincent Lannelongue}{cemef}
\icmlauthor{Elie Hachem}{cemef}
\end{icmlauthorlist}

\icmlaffiliation{cemef}{CEMEF - Mines Paris PSL}

\icmlcorrespondingauthor{Paul Garnier}{paul.garnier@minesparis.psl.eu}

\icmlkeywords{Machine Learning, CFD, GNN}

\vskip 0.3in
]



\printAffiliationsAndNotice{\icmlEqualContribution} 

\begin{abstract}
Machine Learning surrogates for Computational Fluid Dynamics (CFD), particularly Graph Neural Networks (GNNs) and Transformers, have become a new important approach for accelerating physics simulations. However, we identify a critical bottleneck in the field: while architectures have advanced significantly, the \crblue{common} underlying training paradigms remain bound to naive assumptions, such as node-wise supervision and explicit Euler time-stepping. These legacy choices ignore the stiff dynamics and local flux continuity inherent to numerous partial differential equations resolution methods, such as Finite Element, Difference, or Volume (FEM). In this work, we propose a unified framework to bridge the gap between geometric deep learning and rigorous numerical analysis. We introduce three key innovations: (1) Multi Node Prediction, a \crblue{stencil-level} objective that predicts field values for a node's full local topology, enforcing spatial derivative consistency; (2) Temporal Correction, replacing unstable explicit schemes with a predictor-corrector via temporal Cross-Attention; and (3) Geometric Inductive Biases, leveraging 3D Rotary Positional Embeddings (RoPE) to robustly capture rotational symmetries in unstructured meshes. We evaluate this framework across three architectures (MeshGraphNet, Transolver, and a Transformer) on diverse physics datasets. Our approach yields consistent improvements in accuracy and stability, particularly in long-horizon rollouts, while producing latent representations that generalize to unseen subtasks such as Wall Shear Stress or Pressure prediction. Code is available at \url{https://github.com/DonsetPG/graph-physics}.
\end{abstract}

\section{Introduction}
\label{sec:introduction}
Simulating physical phenomena, particularly CFD, consist of solving Partial Differential Equations (PDEs) over complex geometries discretized as unstructured meshes. While traditional solvers are made of large-scale linear algebra kernels, they suffer from a fundamental inefficiency: every new simulation starts from scratch, ignoring the potentially useful data from previous runs. This computational bottleneck has motivated the rapid adoption of Machine Learning (ML) as a reusable surrogate for physics simulation.

Early data-driven approaches targeted structured grids using Convolutional Neural Networks (CNNs) \cite{Thuerey2018} or Physics-Informed Neural Networks (PINNs) \cite{RAISSI2019686}. However, the geometric complexity of real-world engineering necessitates unstructured discretizations. Notably, Graph Neural Networks (GNNs) and Message-Passing Systems (MPS) \cite{Battaglia2018, pfaff2021learning} have emerged as the dominant paradigm, enabling predictions on irregular domains ranging from weather forecasting \cite{lam2023graphcastlearningskillfulmediumrange} to hemodynamics \cite{Suk_2024}. To handle scaling, recent works have integrated FEM-inspired multigrid strategies \cite{Fortunato2022, garnier2024multigridgraphneuralnetworks} or adopted Transformer-based architectures capable of global attention \cite{wu2024transolverfasttransformersolver}.

Despite these architectural advances, the fundamental \emph{training methodology} has remained largely unchanged since the creation of MeshGraphNet. We argue that the community has over-optimized network architectures while neglecting the numerical priors required to solve PDEs effectively. We identify two specific failures in the current state-of-the-art:

\begin{enumerate}
    \item \textbf{Node-wise prediction:} Standard losses minimize error per node in isolation. In FEM, however, physics is defined by the flux across element boundaries. Predicting a single node ignores the local differential information essential for conservation laws.
    \item \textbf{Explicit Euler scheme:} \hypertarget{upd:euler-claims}{}\crblue{Many common discrete-time} surrogates update the state via $\mathbf{u}_{t+\Delta t} = \mathbf{u}_{t} + \Delta t \Phi(\mathbf{u}_{t})$. This mimics an explicit Euler scheme, \crblue{which can be} numerically unstable for stiff dynamics and prone to error accumulation over long trajectories.
\end{enumerate}

In this paper, we challenge these legacy choices. We propose a methodology that aligns ML surrogates with the principles of rigorous numerical solvers. We argue that a surrogate must predict not just the value at a point, but the local stencil of the solution, and that temporal evolution requires attention-based integration to handle stiff timesteps. We present the different contributions of this paper below.

\begin{itemize}
    \item Multi Node Prediction: We introduce a training task where the model predicts the next state for a node \emph{and} its topological neighbors. This acts as a regularizer that enforces local smoothness and continuity, similar to flux reconstruction in FEM.
    \item Temporal Correction: We replace the explicit Euler residual connection with a multi-step Cross-Attention mechanism. This allows the model to approximate an implicit time-stepping scheme by leveraging spatial predictions, thereby significantly improving stability.
    \item Complete Ablation study: We evaluate multiple improvements of standard architecture, including 3D Rotary Positional Embeddings, attention variants, activation function, and loss objectives. \ablation{We present them in orange}.
    \item Model-Agnostic Evaluations: We validate these contributions across three distinct architectures (MeshGraphNet, Transolver, Transformer) and three \crblue{primary} datasets, \crblue{with additional large-scale, long-range, and geometry-shift stress tests,} demonstrating that our improvements scale across both model size and physical complexity.
\end{itemize}

\subsection{Related Work}

\paragraph{Supervsied loss for GNNs} While a large amount of methods now improve the supervised loss using Physics Informed losses \cite{Cai2021,Liu2025}, we aim our focus towards supervised 
tasks. \cite{belkin2006manifold,zhang2018link} defined supervised loss where the model is trained to predict properties of adjacent nodes (e.g., edge existence or edge labels) using losses defined over $(u,v)\in E$, often implemented as link prediction or pairwise classification with negative sampling. A closely related approach \cite{alsentzer2020subgraph} reframes node learning as subgraph learning: each labeled node induces a 1-hop enclosing subgraph; the GNN produces a representation of the induced neighborhood via pooling; and supervision is applied at the subgraph level rather than directly on the center node embedding. However, most of those approaches are unfit for node regression tasks. A closely related body of work is the study of multiple-token predictions in NLP. \cite{gloeckle2024betterfasterlarge} introduced a supervised task to predict multiple future tokens, thereby shifting learning from single-instance targets to richer, structured targets that better capture local context and dependencies. Later on, a similar approach was used by DeepSeek in \cite{deepseekai2025deepseekv3technicalreport}.

\paragraph{Temporal Schemes for Machine Learning surrogate}
In many classical setups, one can decouple the learned spatial operator \(f_\theta\) (or learned increment \(\Delta u_\theta\)) from the temporal integrator, and then advance the state with higher-order explicit schemes (e.g., midpoint \cite{butcher2016numerical}, Runge-Kutta2 (RK2) \cite{Runge1895,kutta1901beitrag}, RK3/RK4, or multistep Adams-Bashforth \cite{hairer1993solving}) by evaluating the surrogate multiple times per step. \hypertarget{upd:rw-fen}{}\crblue{Several learned simulators also depart from residual next-step updates through continuous-time or solver-coupled formulations. Finite Element Networks, for instance, derive a continuous-time model from FEM and estimate latent dynamics on mesh cells before advancing them with a finite-element mass matrix \cite{lienen2022finiteelementnetworks}.}

\paragraph{\crblue{Long-range interactions on meshes}}
\hypertarget{upd:rw-longrange}{}\crblue{Oversquashing and long-range dependency bottlenecks have also been studied directly in mesh simulators. PIORF rewires mesh graphs using physics-informed Ollivier-Ricci flow to connect bottleneck regions with high-gradient nodes \cite{yu2025piorf}; Hamiltonian-based graph simulators target information preservation during long-range propagation \cite{hoang2026hamiltonian}; and HCMT uses hierarchical contact transformers to propagate collision effects across distant flexible-body regions \cite{yu2024hcmt}. EAGLE further stresses long-range turbulent dynamics with moving sources and varying geometries \cite{janny2023eagle}.}

\paragraph{Positional Encoding}
For 2D imagery, absolute positional encodings adapt 1D sinusoidal features to a grid by factorizing row and column coordinates and either summing or concatenating the resulting embeddings. Many Vision Transformers (ViTs) adopt learned absolute embeddings over patches \cite{dosovitskiy2021imageworth16x16words}, which can be effective but may generalize poorly across resolutions without interpolation. In contrast, relative positional encodings represent pairwise offsets between tokens (e.g., $\Delta x, \Delta y$), enabling attention weights to depend on relative geometry rather than absolute location \cite{shaw2018selfattentionrelativepositionrepresentations,wu2021rethinkingimprovingrelativeposition}. In the case of points clouds, a common strategy injects coordinates directly by passing $(x,y)$ or $(x,y,z)$ through an MLP and adding/concatenating the result to point features \cite{wu2022pointtransformerv2grouped}. However, raw absolute coordinates can entangle pose with content; thus, many approaches emphasize \emph{relative} or \emph{local} encodings constructed from neighborhoods (e.g., $ p_j-p_i$ for neighboring points), distances, and angles \cite{zhao2021pointtransformer}. Regular positional encoding \cite{vaswani2023attention} has also been adapted to 3D points and graphs \cite{alkin2025abuptscalingneuralcfd}. In the case of graphs, one can represent positional encoding with Laplacian Eigen Vectors \cite{dwivedi2021generalizationtransformernetworksgraphs}, learnable positional Encoding \cite{kreuzer2021rethinkinggraphtransformersspectral}, and RandomWalk \cite{dwivedi2022graphneuralnetworkslearnable}. 

\section{Preliminaries}
\label{sec:preliminaries}

\subsection{Mesh as Graph}
\label{subsec:graph}

We consider a mesh as an undirected graph $\mathcal{G} = (\mathcal{V},\mathcal{E})$. 
$\mathcal{V} = \{\mathbf{x}_i\}_{i=1:N}$ is the set of vertices or nodes, where each $\mathbf{x}_i \in \mathbb{R}^{p}$ represents the attributes of node $i$.
$\mathcal{E} = \{\left(\mathbf{e}_k, r_k, s_k\right)\}_{k=1:N^e}$ is the set of edges, where each $\mathbf{e}_k$ represents the attributes of edge $k$, $r_k$ is the index of the receiver node, and $s_k$ is the index of the sender node. 

For certain models (Transformer and Transolver), we omit edge attributes and treat each node as a token. We note $\mathbf{X} = (\mathbf{x}_1, \mathbf{x}_2,...,\mathbf{x}_N)^\top \in \mathbb{R}^{N \times p}$ our input matrix, made of $N$ tokens of dimension $p$. Let $\mathbf{Z} = (\mathbf{z}_1, \mathbf{z}_2,...,\mathbf{z}_N)^\top \in \mathbb{R}^{N \times d}$ be the $d$-dimensional embedding of our nodes. We define $\mathbf{A}$ as the adjacency matrix of our graph, setting $\mathbf{A}_{ij} = \mathbf{A}_{ji} = 1$ if $(i,j) \in \mathcal{E}$ and $0$ otherwise.

\subsection{Models}
\label{subsec:model}

In the remainder of the paper, we study three different models: MeshGraphNet \cite{pfaff2021learning}, Transformer \cite{garnier2025trainingtransformersmeshbasedsimulations}, and Transolver \cite{wu2024transolverfasttransformersolver}. All models follow an Encode-Process-Decode architecture similar to \cite{sanchezgonzalez2020learning}. The Encoder maps the input nodes into a latent space. We then apply $L$ layers of a processor. Finally, the Decoder maps back our outputs into a meaningful space. At each step, our models are auto-regressive, meaning that their output is used as input for the next step of the simulation. In the remainder of the section, we briefly present how Encoders and Decoders are built before introducing the basic concepts of our three models.

\paragraph{Encoder and Decoder}
\label{para:encoder-decoder}

We use a simple two-layer MLP to encode our inputs. Our encoder $\mathcal{E} : \mathbb{R}^p \rightarrow  \mathbb{R}^d$ maps our nodes (resp. our edges) $\mathbf{X} \in \mathbb{R}^{N \times p}$ into a latent space $\mathbf{Z} \in \mathbb{R}^{N \times d}$. The parameter $d$ is shared across all our layers as the main width parameter. The Decoder $\mathcal{D} : \mathbb{R}^d \rightarrow  \mathbb{R}^{2\text{ or }3}$ generates a physical output (in 2 or 3D) from the latent space using two layers as well. 

\paragraph{MeshGraphNet}
\label{para:mgn}

Recall that each spatial processor in MeshGraphNet are message-passing layers that update both the node and edge attributes given the current node and edge attributes:

\begin{equation}
\vspace{-1mm}
\begin{alignedat}{3}
        &\mathbf{e}_k' &&= f^e(\mathbf{e}_k,\mathbf{z}_{r_k},\mathbf{z}_{s_k}) &&\hspace{1em} \forall k \in E
        \label{eq:edge_model} \\
        &\bar{\mathbf{e}}_r' &&= \displaystyle\sum_{e \in E_r'} e &&\hspace{1em} \forall r \in V \\
        &\mathbf{\Tilde{z}}_r &&= [\mathbf{z}_r,\bar{\mathbf{e}}_r'] &&\hspace{1em} \forall r \in V \\
        &\mathbf{z}_r' &&= f^v(\mathbf{\Tilde{z}}_r) &&\hspace{1em} \forall r \in V
\end{alignedat}
\end{equation}

where $f^e$ and $f^v$ are simple Multi-Layer Perceptron (MLP). 

\paragraph{Transformer}
\label{para:transformer}

Each spatial processor from Transformer is defined using $\text{MHA}(\mathbf{Z},\mathbf{A}) = \left(
\text{softmax}\Big( \frac{\mathbf{A}\odot(QK^T)}{\sqrt{d}} \Big)\right)V$ where $Q, K, V$ are linear projections of $\mathbf{Z}$ and $\odot$ is the Hadamard product. \hypertarget{upd:jumpers}{}\crblue{In our Transformer baseline, $\mathbf{A}$ also includes random long-distance edges, or jumpers, to support nonlocal information exchange; graph-rewiring strategies such as PIORF could be combined with the proposed modules.} The overall architecture is then defined as: 

\begin{align}
\vspace{-1mm}
    \mathbf{Z'}_l &= \text{RMSNorm}\big(\text{MHA}(\mathbf{Z}_{l-1}, \mathbf{A}) + \mathbf{Z}_{l-1}\big) && \ell \in [1\ldots L] \label{eq:msa_apply} \\
    \mathbf{Z}_l &= \text{RMSNorm}\big(\text{GatedMLP}(\mathbf{Z'}_{l}) + \mathbf{Z'}_{l}\big) && \ell \in [1\ldots L] \label{eq:mlp_apply} \\
\end{align}

where GatedMLP is a Gated Multi-Layer Perceptron \cite{dauphin2017language}.

\paragraph{Transolver}
\label{para:transolver}

The spatial processors from Transolver learn $M$ latent slices from the $N$ vertices of the mesh and process them using a Physics-Attention operator. The processor then broadcasts the information back to the mesh. The processors can be summarized as follows:

\begin{align}
\vspace{-1mm}
\mathbf{z}_j &= \frac{\sum_{i=1}^N \mathbf{w}_{i,j} \mathbf{x}_i}{\sum_{i=1}^N \mathbf{w}_{i,j}}, \quad \text{with} \quad \mathbf{w} = \text{Softmax}(\text{MLP}(\mathbf{x})) \\
\mathbf{x}'_i &= \sum_{j=1}^M \mathbf{w}_{i,j} \cdot \text{Attention}(\mathbf{z})_j
\end{align}

\begin{remark}

The last two architectures are based on an attention mechanism. \ablation{We also investigate in the appendix several modifications to regular attention: Multi Head Latent Attention \cite{liu2024deepseek} and variants of Gated Delta Nets and Linear Attention \cite{qiu2025gatedattentionlargelanguage,kimiteam2025kimilinearexpressiveefficient}. We add those modifications to our ablation study.}

\end{remark}

\section{Methodology}
\label{sec:methodo}

\subsection{Multi Node Prediction}
\label{subsec:mnp}

\begin{figure}[!t]
  \centering
  \includegraphics[width=3.3in]{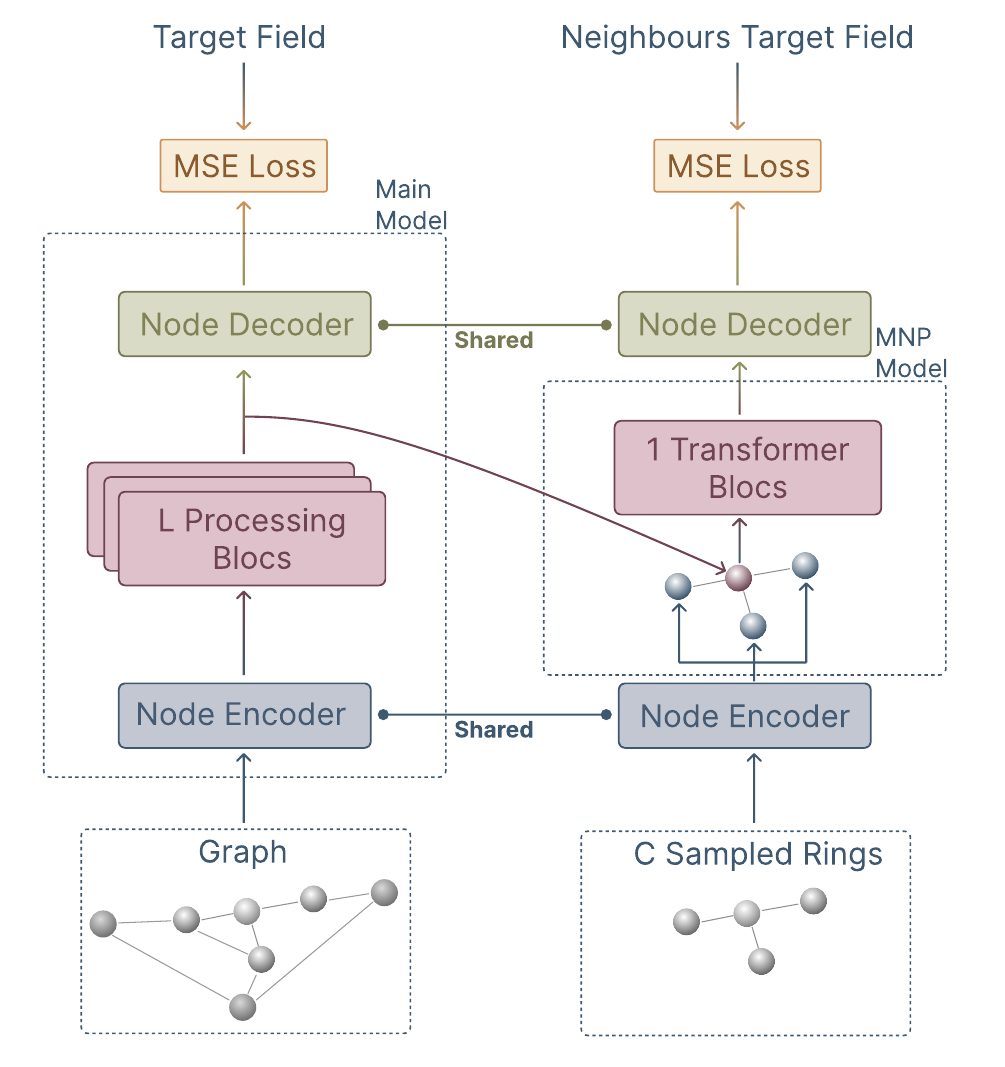}
\caption{\small\textbf{Multi Node Prediction} We process each node with a given model. Before decoding each node, we construct rings that consist of a latent node and freshly encoded neighbors. We then train a small cross-attention layer to predict the fields of each neighbor, while most relevant information lives in the latent central node.
  }
  \label{fig:MTP}
\end{figure}

Recall that \crblue{many common} approaches use a regression task on a per-node basis, where $\mathbf{z^L_i}$ represents the embedding of node $i$ after $L$ layers:

\begin{equation}
\mathcal{L}_{\text{main}}(\theta)=\frac{1}{N}\sum_{i=1}^N \ell(\mathcal{D}(\mathbf{z^L_i}), y_i)
\end{equation}

which does not explicitly enforce that the latent representation of $\mathbf{z^L_i}$ carries any information about the neighborhood (or stencil) around $i$. On irregular meshes with long-range, anisotropic dynamics, this omission manifests in several potential pathologies: (i) a lack of relevant information related to flux reconstruction in FEM; (ii) oversquashing, since relevant neighbor information must be compressed through message passing instead of living inside each node directly; (iii) stability issues, where small inconsistencies on a stencil around $i$ can amplify over rollouts, instead of having the ability to check itself against its neighbors. 

To that end, we make a small but effective modification to our supervised task: Multi-Node Prediction. Let $C\!\subset\!V$ be a set of centers sampled per step and $\mathcal{N}(i)$ the 1‑hop neighborhood of $i$. \hypertarget{upd:mnp-sampling}{}\crblue{Each center is sampled uniformly from internal nodes only; boundary nodes are excluded as centers, the number of centers is fixed per batch, and centers are resampled at each training step.} We form a \emph{star sequence}:

\begin{equation}
\mathbf{S_i} \;=\; \big[\mathbf{z^L_i} \;\;\; \mathbf{z^0_{j_1}} \;\;\; \dots \;\;\; \mathbf{z^0_{j_{|\mathcal{N}(i)|}}}\,\big]
\end{equation}

where $\mathbf{z^0_{j_k}} = \mathcal{E}(\mathbf{x^0_{j_k}})$ is the encoded feature of node $j_k$, a neighbor of node $i$, represented by its latest latent version $\mathbf{z^L_i}$. The goal is to force the latent representation of node $i$ to predict its own next step value, \textbf{but also} the one of its neighbors, only from their encoded values and its own final latent representation. We define a small one-layer transformer $T_{\theta'}$ using ring self-attention (neighbors attend to the center and to each other, but stars are isolated and do not attend to themself) to yield:

\begin{equation}
    \mathbf{O_i} \;=\; T_{\theta'}(\mathbf{S_i})\in\mathbb{R}^{(|\mathcal{N}(i)|\!+\!1)\times d}
\end{equation}

We then compute the same supervised loss as before between each decoded neighbor (using the same decoder $\mathcal{D}$), and its target field:

\begin{equation}
\label{eq:spMTP}
\mathcal{L}_{\text{MNP}}(\theta, \theta ')
\;=\;
\frac{1}{|C|}\sum_{i\in C}\frac{1}{|\mathcal{N}(i)|}\sum_{j\in\mathcal{N}(i)}
\ell(\mathcal{D}(\mathbf{O_{ij}}), y_j)
\end{equation}

and define the total objective $\mathcal{L}=\mathcal{L}_{\text{main}}+\alpha\,\mathcal{L}_{\text{MNP}}$ with $\alpha = 0.2$. The overall architecture is presented in \autoref{fig:MTP}. Overall, we shift $\mathbf{z^L_i}$ from being only predictive of $y_i$ to being jointly predictive of $\{ y_i ,\{y_j\}_{j\in\mathcal{N}(i)}\}$. More precisely, in \autoref{thm:sobolev}, we prove that patch reconstruction accuracy controls a discrete gradient (hence flux) error.

\hypertarget{upd:eq-format}{}\begin{equation}
\begin{aligned}
\frac{1}{N}\sum_{i=1}^N
\big\| \widehat{\nabla}_h u(\mathbf{x}_i) - \nabla u(\mathbf{x}_i) \big\|^2
&\;\le\;
\frac{C}{h^2}\Big(\mathcal{L}_{\mathrm{MNP}}+\mathcal{L}_{\mathrm{node}}\Big)
\;+\;
\\&
C\,h^2\,\|u\|_{C^2(\Omega)}^2
\end{aligned}
\end{equation}

This proves that $\mathcal{L}_{\mathrm{MNP}}$ is a principled proxy for a Sobolev-type regularization that aligns the learned representation with the PDE's spatial operator.

\begin{algorithm}[h]
\caption{Multi-Node Prediction}
\label{alg:mnp}
\begin{algorithmic}[1]

\State Sample centers $C \subseteq V_{\mathrm{int}}$ \crblue{uniformly from internal nodes} with $|C|=m$
\ForAll{$i \in C$}
    \State $\mathcal{N}(i) \gets \{j \in V : (i,j)\in E\}$ \Comment{1-hop}
    \State Order $\mathcal{N}(i)$ as $(j_1,\dots,j_{K_i})$ where $K_i=|\mathcal{N}(i)|$
    \State $S_i \gets [\,\mathbf{z}^L_i;\ \mathbf{z}^0_{j_1};\ \dots;\ \mathbf{z}^0_{j_{K_i}}\,]$
    \Comment{star sequence}
\EndFor

\State $O_i \gets T_{\theta'}(S_i)$
\Comment{use a block-diagonal attention mask so that different stars do not attend}
\ForAll{$i \in C$}
    \For{$k=1$ \textbf{to} $K_i$}
        \State $j \gets j_k$
        \State $\hat y_{j\mid i} \gets \mathcal{D}(O_{i,k+1})$
        \Comment{token $k{+}1$ corresponds to neighbor $j_k$ (token 1 is the center)}
        \State $\mathcal{L}_{\mathrm{MNP}} \gets \mathcal{L}_{\mathrm{MNP}} + \frac{1}{|C|\,K_i}\,
        \ell\!\big(\hat y_{j\mid i},\,y_j\big)$
    \EndFor
\EndFor
\end{algorithmic}
\end{algorithm}

The computational overhead is only minor, since we simply add a single cross-attention layer and re-use the encoder $\mathcal{E}$ and decoder $\mathcal{D}$. Having hooks within the training framework also allows for most operations to be run only once. To keep our addition lightweight even on very large meshes, we pack the star sequences in a vectorized fashion with padding, capping the number of neighbors per center to $K$. Our additional module is used only during training, resulting in a 5\% increase in training time. During our ablation study, we only examine the impact of this new task on performance, the effect of the number of centers we randomly select at each step (including the increase in training time), and the impact of this task on the nodes' latent representations. \ablation{We add this method in our ablation study, and both study the impact of the number of centers chosen and other supervised approaches, such as an $L_2$ loss on the field's gradient.}

\subsection{Temporal Correction}

Recall that \crblue{many discrete-time machine learning surrogates generate} next-step predictions with the operation $\mathbf{u}_{t+1} = \mathbf{u}_{t} +\Phi(\mathbf{u}_{t})$, where $\Phi$ is made of $L$ spatial processing layers: $\mathbf{Z}^{\ell+1} =\mathbf{Z}^\ell + \phi_\ell(\mathbf{Z}^\ell)$. We argue that one should actually see the $L$ spatial updates as $L$ predictor-corrector \cite{Gragg1964} layers, \textit{i.e.} as $L$ discretizations of the interval $[t,t+\Delta t]$, with the following formulation:

\vspace{-3mm}
\begin{equation}
\begin{aligned}
\tilde{\mathbf{Z}}^{\,\ell+1}
&= \mathbf{Z}^{\ell} + \phi_\ell^s(\mathbf{Z}^{\ell})
\\
\mathbf{Z}^{\ell+1}
&= \mathbf{Z}^{\ell} + \phi_\ell^t\!\big(\tilde{\mathbf{Z}}^{\,\ell+1},\,\mathbf{Z}^{\ell}\big)
\end{aligned}
\label{eq:pc}
\end{equation}
\vspace{-6mm}

where $\phi_\ell^s$ is the same (spatial) processor as defined before, specific to each architecture, and $\phi_\ell^t$ is a learnable temporal corrector. Importantly, we do not increase the temporal context of our model; we still use a single timestep to predict the next. However, we interpret each processor block $\ell$ as a predictor in $[t +\Delta t\frac{\ell}{L}, t + \Delta t\frac{\ell+1}{L}]$. Equation~\eqref{eq:pc} separates space from time while keeping a residual identity path that preserves gradients and reduces drift across rollouts.

Anisotropic PDEs require directionally selective stencils. While standard approaches are often largely isotropic, we implement our temporal corrector using cross-attention with queries from the predicted state, keys and values from the previous state, thereby routing updates along directions encoded by the evolving field itself. To overcome classical GNN pathologies such as oversquashing and oversmoothing, we also introduce a gating mechanism in the temporal corrector. 

Let $\mathbf{C} = [\tilde{\mathbf{Z}}^{\,\ell+1},\,\mathbf{Z}^{\ell}]$. Let $G([\tilde{\mathbf{Z}},\mathbf{Z}]) \in [0,1]^N$ be a gating mechanism implemented as a 2-layer MLP with a softmax for the final activation function. Let $M\!\big([\mathbf{h}^{\ell+1},\mathbf{h}^{\ell}]\big)$ be a mixing MLP with 2 layers. 
We implement the temporal corrector as follows: 

\begin{equation}
\begin{aligned}
\eqnode{L1}{\tilde{\mathbf{Z}}^{\,\ell+1}}
  &\eqnode{R1}{= \mathbf{Z}^{\ell} + \phi_\ell^s(\mathbf{Z}^{\ell})}
\\[1.5em] 
\eqnode{L2}{\bar{\mathbf{Z}}^{\,\ell+1}}
  &\eqnode{R2}{= G(\mathbf{C}) \odot
      \text{CA}(\mathbf{C})}
\\
\eqnode{L3}{\phi_\ell^t\!\big(\mathbf{C}\big)}
  &\eqnode{R3}{= \bar{\mathbf{Z}}^{\,\ell+1} + M(\mathbf{C})}
\\[1.5em]
\eqnode{L4}{\mathbf{Z}^{\ell+1}}
  &\eqnode{R4}{= \mathbf{Z}^{\ell} + \phi_\ell^t\!\big(\mathbf{C}\big)}
\end{aligned}
\label{eq:pc2}
\end{equation}

\begin{tikzpicture}[remember picture, overlay]
  \definecolor{myfill}{HTML}{C3C8D1}
  \definecolor{myline}{HTML}{3E586F}
  \definecolor{myfillred}{HTML}{DEC1CA}
  \definecolor{mylinered}{HTML}{704351}

  \tikzset{
    mybox/.style={
      draw=myline,
      dashed,
      rounded corners,
      inner sep=5pt,
      fill=myfill,
      fill opacity=0.15,
      thick
    },
    mylabel/.style={
      anchor=east,
      font=\sffamily\small\bfseries,
      xshift=-0.5cm,
      text=myline
    }
  }

  \node[mybox, fit=(L1)(R1)] (b1) {};
  \node[mylabel, align=right] at (b1.west) {Spatial};

  \node[mybox, draw=mylinered, fill=myfillred, fit=(L2)(R2)(L3)(R3)] (b2) {};
\node[mylabel, text=mylinered, align=right] at (b2.west) {Temporal\\Cross Att.};

  \node[mybox, fit=(L4)(R4)] (b3) {};
  \node[mylabel, align=right] at (b3.west) {Correction};

\end{tikzpicture}

where $\text{CA}$ is cross-attention and $\odot$ is the Hadamard product. During our ablation study, we investigate both the impact of the gating and mixing mechanism, and the trade-off between the frequency of said temporal correction and the increase in training time. 
In \autoref{thm:theta_stability}, we show that the temporal correction block enlarges the class of stable one-step maps that the network can realize compared to an explicit residual update, thus \crblue{allowing it to realize a larger class of stable one-step maps under the theorem assumptions than} the standard approach. \ablation{We add a temporal corrector in our ablation study, and investigate frequently to apply it and the impact of the gating and mixing mechanism.}

\subsection{3D RoPE}

\begin{figure*}[ht!]
  \centering
  \includegraphics[width=0.99\textwidth]{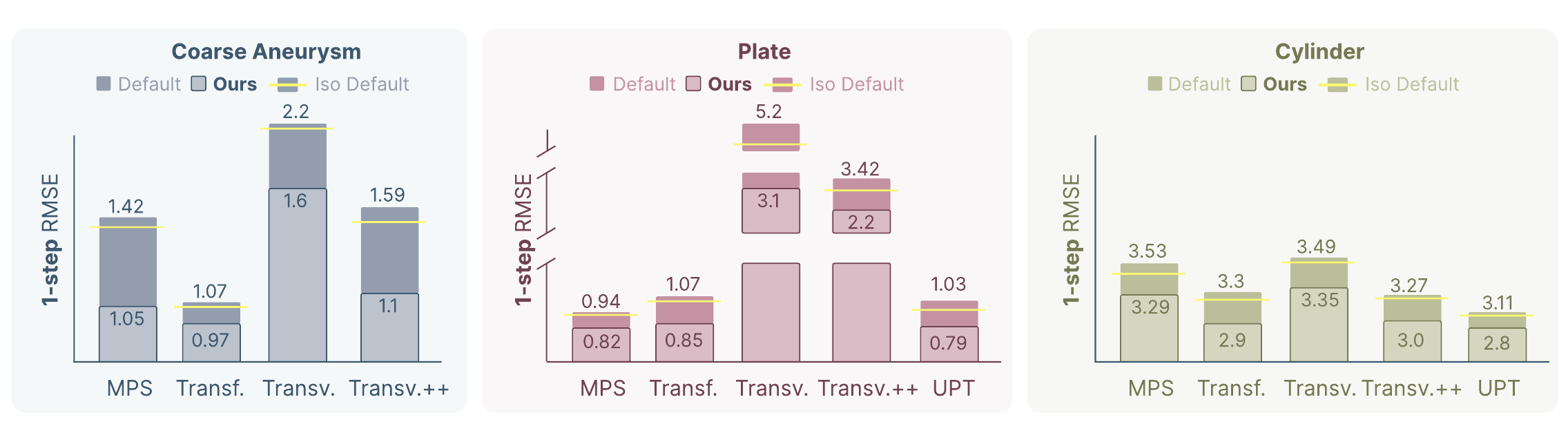}
  \caption{\small\textbf{Results on 1-step RMSE.} We see improvements for all models and all datasets. Since our approach incurs a slight increase in the number of trainable parameters, we also compute the metric for the original architecture with the same number of parameters as ours (represented by the yellow lines).
  }
  \label{fig:improv-dataset}
\end{figure*}

Irregular meshes exhibit strong, orientation‑dependent transport (e.g., advection along streamlines, anisotropic diffusion along fiber directions). Message‑passing baselines and geometry‑agnostic attention tend to learn isotropic mixing, hurting long‑range fidelity. While using relative position in edges, or absolute position in nodes, can mitigate those issues, we argue that they should work in relation to another sort of positional encoding.

Injecting 2D/3D rotary structure into queries Q and keys K makes attention explicitly sensitive to relative 2D/3D offsets, which improves directional selectivity and long‑horizon accuracy in our rollouts while preserving efficiency for neighbor‑only graphs.

RoPE applies an axis‑wise orthogonal rotation to Q and K with angles proportional to position. For one 2D channel pair $(u,v)$ and scalar angle $\theta$, $\mathrm{Rot}_\theta[u,v]=[u\cos\theta - v\sin\theta,\; u\sin\theta+v\cos\theta]$. 
If $R(\theta_i)$ and $R(\theta_j)$ are rotations applied to $q_i$ and $k_j$, then 
\[
\langle R(\theta_i)q,\, R(\theta_j)k\rangle \;=\; \langle q,\, R(\theta_j-\theta_i)k\rangle
\]
so the score depends on \emph{relative} angle. Extending to 3D, we assign disjoint channel pairs to the $x$, $y$, and $z$ axes and rotate with angles $\theta^{(a)}_{i,r}=\omega_r\,\tilde{p}_{i,a}$, where $\tilde{p}_i=p_i-\bar{p}$ are centered coordinates and $\{\omega_r\}$ are multi‑scale frequencies. The dot‑product remains a function of per‑axis \emph{relative} offsets $(\tilde{p}_j-\tilde{p}_i)$, which (i) encodes oriented distances that match anisotropic PDE fluxes and (ii) is translation‑invariant by construction. 

For each head of width $d_h$, we allocate channel pairs across the three axes. Let $\mathcal{I}_x,\mathcal{I}_y,\mathcal{I}_z$ index those pairs. With centered coordinates $\tilde{p}_i=(\tilde{x}_i,\tilde{y}_i,\tilde{z}_i)$ and frequencies $\{\omega_r\}$, we define:
\[
\begin{aligned}
&\forall (2r,2r{+}1)\!\in\!\mathcal{I}_x:\; \\&
\begin{bmatrix}\tilde{q}_{i,2r}\\ \tilde{q}_{i,2r{+}1}\end{bmatrix} =
\begin{bmatrix}\cos(\omega_r \tilde{x}_i) & -\sin(\omega_r \tilde{x}_i)\\ \sin(\omega_r \tilde{x}_i) & \cos(\omega_r \tilde{x}_i)\end{bmatrix}
\begin{bmatrix}q_{i,2r}\\ q_{i,2r{+}1}\end{bmatrix}\\
&\text{and similarly for } \mathcal{I}_y,\mathcal{I}_z, \text{and the keys } \tilde{k}_{j}
\end{aligned}
\]

We then compute attention using $(\tilde{q}_i,\tilde{k}_j)$. This preserves neighbor sparsity, adds no parameters, and composes cleanly with all attention-based operations.

\ablation{During our ablation study, we investigate whether RoPE helps with the performance or not. We also compare it to several other positional encoding approaches: a learnable positional embedding, a learnable relative position bias, and weighting the adjacency matrix by nodes' Euclidean distances.}

\section{Experiments}

\paragraph{Datasets} We evaluated our models across different use cases. The first dataset is the flow past a cylinder \cite{pfaff2021learning} simulated using the COMSOL solver. The second dataset is the flow inside a brain aneurysm \cite{aneurysmdataset}, simulated using the CIMLIB \cite{cimlib} solver. Finally, the last dataset is the DeformingPlate \cite{pfaff2021learning}. Details such as the attributes used and the simulation time step $\Delta t$ are shown in Table \ref{tab:datasets-details} and \autoref{fig:datasetsdetails}. \hypertarget{upd:datasets-main}{}\crblue{Cylinder and DeformingPlate use the full MeshGraphNet configuration with 1000 training trajectories and 100 test trajectories; Coarse Aneurysm uses 100 training trajectories and 20 test trajectories. The trajectories contain 600, 400, and 80 steps for Cylinder, DeformingPlate, and Aneurysm, respectively; $\Delta t=0.01$s for Cylinder and Aneurysm.} While the Cylinder and DeformingPlate datasets consist of relatively small graphs (between 1 and 2,000 nodes), the brain aneurysm meshes are one order of magnitude larger. \hypertarget{upd:bc}{}\crblue{Following MeshGraphNet, each node is assigned a node type. During prediction, velocity boundary values are enforced for all node types except Normal and Outflow, while pressure is enforced only at Inlet nodes; MNP centers are sampled from internal nodes, although boundary nodes may still appear as neighbors and are handled with the same boundary-value enforcement rules.}

\paragraph{Training procedures}To evaluate our models, we use the 1-step RMSE and the All-Rollout RMSE defined in \cite{sanchezgonzalez2020learning}. \hypertarget{upd:metrics}{}\crblue{Unless stated otherwise, RMSE values are averaged over predicted variables, nodes, timesteps, trajectories, and five random seeds; lower is better, and standard deviations are reported in ablation plots when space permits.} Models were trained for 20 epochs with an AdamW \cite{loshchilov2019decoupledweightdecayregularization} optimizer using $\beta _1 = 0.9, \beta _2 = 0.95$. We used a learning rate schedule with a warm-up period of 1000 steps, cosine decay, and a maximal learning rate of $10^{-3}$. Each model is trained on 5 different seeds. We reproduce the standard deviation obtained from the ablations. We introduce noise to our inputs using the same strategy as \cite{sanchezgonzalez2020learning}.
More specifically, we add random noise $\mathcal{N}(0,\sigma)$ to dynamical variables (see \autoref{tab:noises}) at each training step and train solely on next-step prediction.

\section{Results}

Across the three different datasets and the three different models, we obtain improvements on all rollout metrics between 20 and 30 percent, for only a 10\% increase in both training or inference time, using the optimal configurations defined in the ablation.
(See \autoref{fig:improv-dataset} for 1-step improvements and \autoref{fig:improv} for more general results). We want to emphasize that our approach improves results across highly diverse datasets (2D, 3D, fluids, solids, varying numbers of nodes) and models (message-passing-based GNN, attention-based GNN, attention-based point clouds). Importantly, models that match the total number of parameters in our architecture by increasing their width obtain only minor improvements. We believe this makes our approach a highly efficient method for improving the performance of mesh-based simulation, regardless of the dataset or architecture. 

\hypertarget{upd:ood}{}\crblue{Finally, we tested geometry-shift generalization by training on a single-cylinder configuration and evaluating on multiple cylinders and on different shapes. These supplementary rollouts show that the proposed modules remain beneficial under unseen geometries and new vortex/recirculation patterns, but this test should be interpreted as a geometry-shift evaluation rather than a broad Reynolds-number generalization study.}

\begin{figure*}[ht!]
  \centering
  \includegraphics[width=0.99\textwidth]{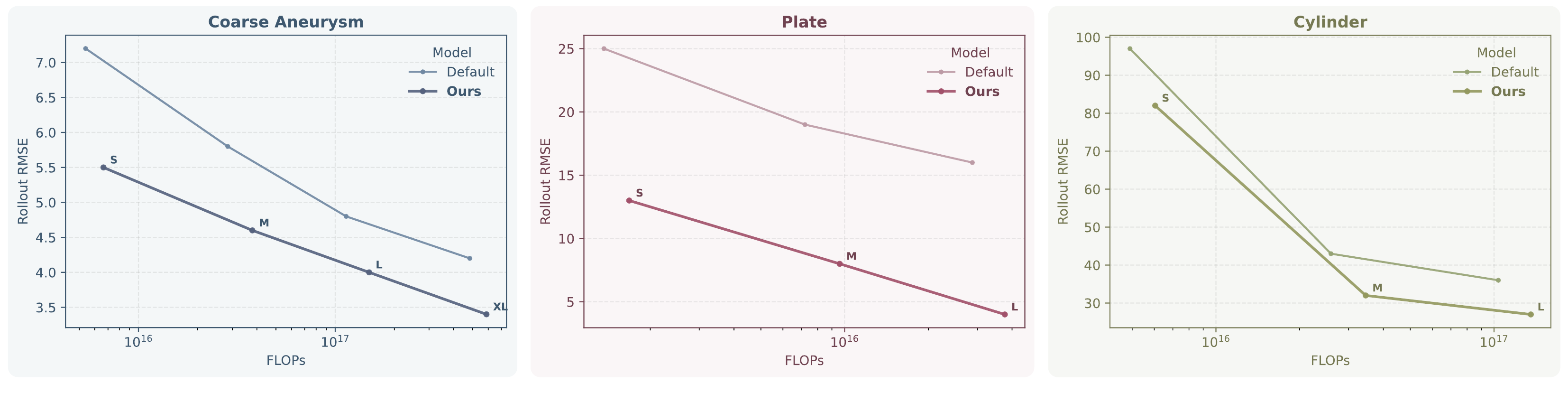}
  \caption{\textbf{Scaling with model size.}
  All rollout RMSE for a Transformer model on the three different datasets. Even when training models larger and larger, our approach keeps scaling similarly to the previous model.}
  \label{fig:scalesize}
\end{figure*}

We also study whether our approach scales with training time and the model's size, or if those improvements decrease as we scale up. In \autoref{fig:scalesize}, we present the results of a Transformer architecture on each dataset, as we scale the parameters from 500k (S), to 3M (M), 14M (L), and 51M (XL). For the Aneurysm and the Deforming Plate dataset, our approach scales with the number of parameters, yielding consistent results. On the Cylinder dataset, while our approach still yields better results, the improvements are less consistent. This may be because of benchmark saturation, as noted in \cite{garnier2025trainingtransformersmeshbasedsimulations}. 

\begin{figure}[th!]
  \centering
  \includegraphics[width=3.3in]{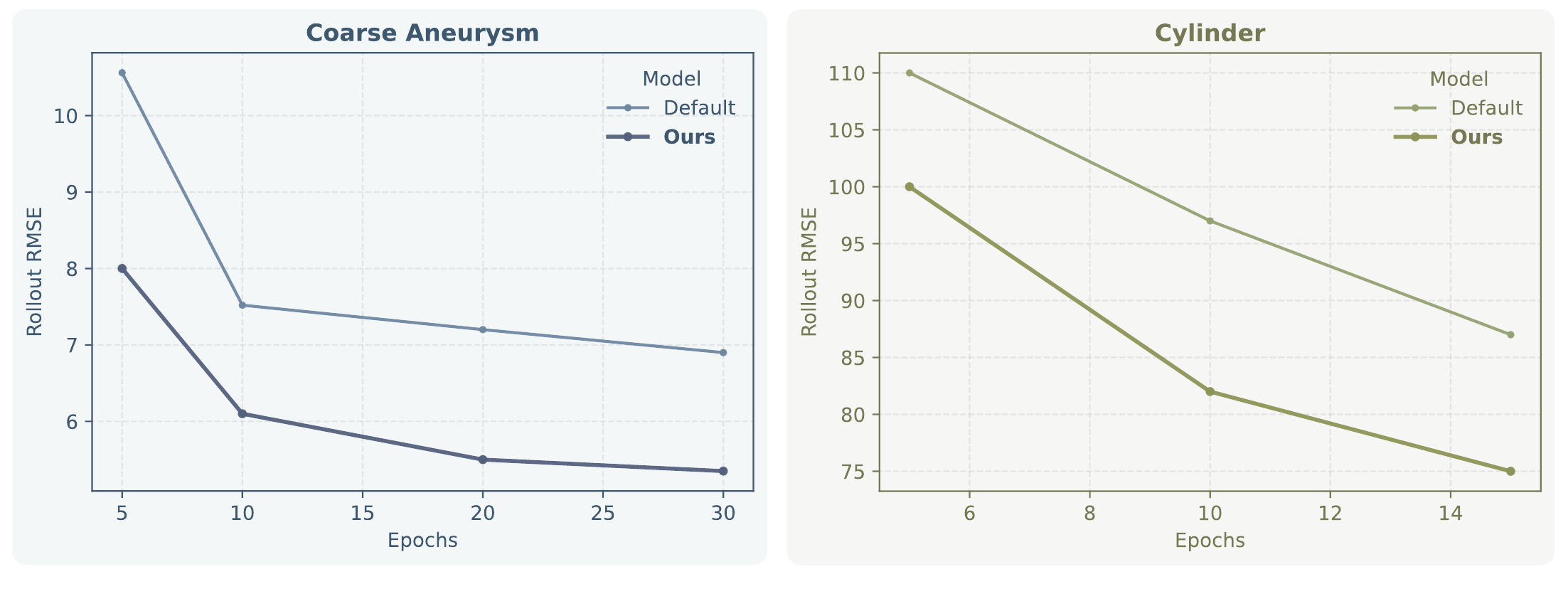}
  \caption{\textbf{Scaling with training time}
  All rollout RMSE for a MeshGraphNet architecture of two different datasets, for different training schedules.}
  \label{fig:scaletime}
\end{figure}

In \autoref{fig:scaletime}, we study a message passing GNN on the Aneurysm and the Cylinder dataset. We scale the training times from 5 to 30 epochs for the Aneurysm dataset, and from 5 to 15 for the Cylinder dataset. Similarly to the previous experiment, we find that our approach robustly scales with training time, yielding consistent improvements across different numbers of gradient descent steps.

\subsection{Ablation Studies}

Our ablation study was run on the Aneurysm dataset, where each graph has, on average, \crblue{20k} nodes, and we used Transformer models. 

\paragraph{Multi Node Prediction} We study the impact of our Multi Node Prediction task, and the impact of the number of centers we use. Overall, we find that using as few as 256 centers (around 2.5\% of the nodes) per trajectory yields improvements on both 1-step and all-rollout RMSE. We also find that improvements scale with the number of centers selected. Using 1024 centers gives twice the improvements from 256 centers. Results are presented in \autoref{fig:ablation}-middle column. We also demonstrate in \autoref{fig:other-ablation} that our method offers much better results than replacing $\mathcal{L}_{\mathrm{MNP}}$ by a supervised loss on the field's gradient $\ell(\nabla \hat y_i, \nabla y_i)$. \hypertarget{upd:mnp-sampling-ablation}{}\crblue{We also tested MNP centers sampled uniformly, biased toward the DeformingPlate contact/constraint region, and biased away from it; the three configurations gave similar improvements, with no statistically significant difference from uniform internal-node sampling.}

\paragraph{Temporal Correction} We find that our temporal correction yields strong improvement over 1-step and all-rollout RMSE. We first study the impact of our gating and mixing mechanism. Only using cross-attention worsens the performance, to a higher degree than using no temporal correction. This highlights the importance of those two architectures. We also study the impact of the number of temporal corrections. We compare using a correction term after each spatial block, or only after the last one. Both improve the performance, and there's a strong trade-off to find between computation time and better performance. However, those extra parameters are better consumed here than in extra spatial computation, as shown in our previous section. Finally, we also tried to add an extra loss term after each temporal correction. That is, after each intermediate $\ell$ step, we decoded the latent graph and used it as an intermediate supervised signal. This doesn't work well at all. Results are presented in \autoref{fig:ablation}-right column.

\paragraph{3D RoPE} Overall, 3D RoPE provides the same performance as using absolute coordinates in the node feature. However, using both at the same time yields strong improvements. On the other hand, any other tested solution provides worse performance. We studied a learned positional embedding, that is, a learnable mapping between absolute position and $\mathbb{R}^d$, which we could sum with the node features. A learnable relative bias, that is, a learnable mapping between differences in absolute position and $\mathbb{R}^d$, we could sum after computing $Q^tK$. Finally, we also tried for the Transformer architecture to replace the unitary adjacency matrix with a weighted version. None of these three approaches yields improvements. Results are presented in \autoref{fig:ablation}-left column.

The rest of the ablation are presented in \autoref{sec:add-exp}.

\begin{figure*}[tbh!]
  \centering
  \includegraphics[width=0.99\textwidth]{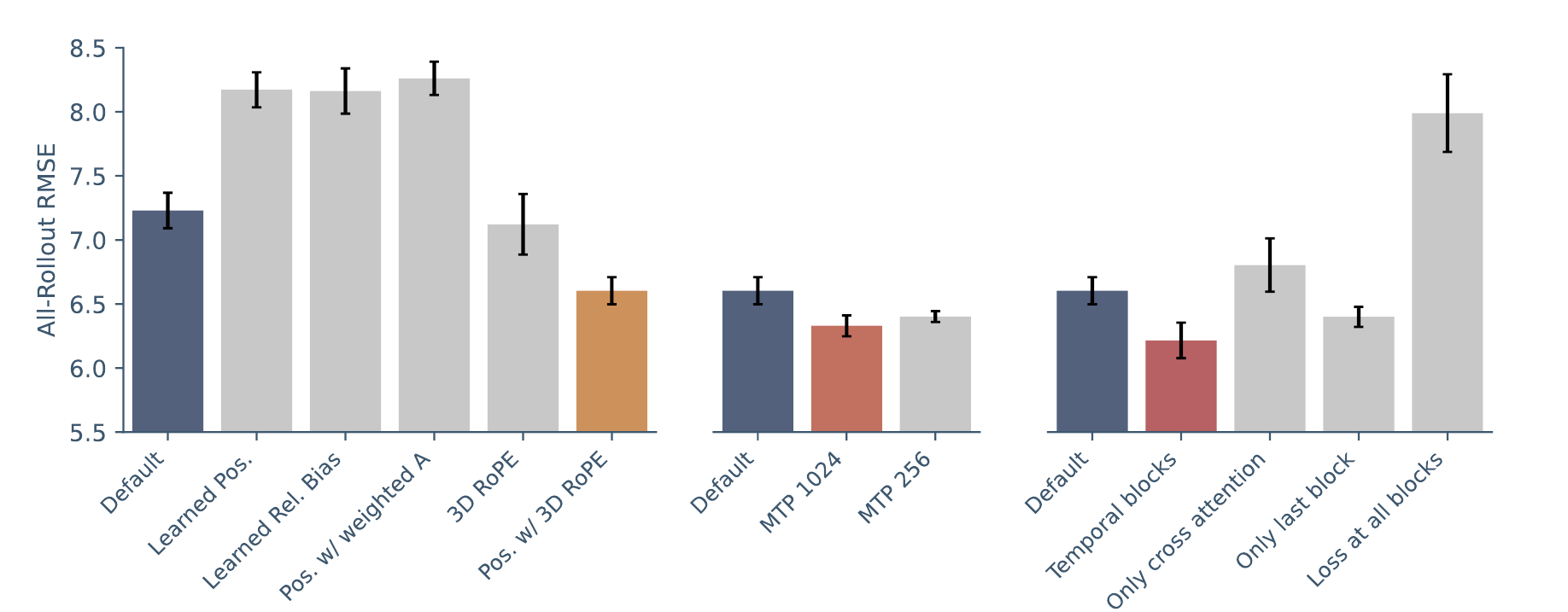}
  \caption{\textbf{Ablation Studies}
  We present our ablation studies for all three improvements and their alternatives. Ablations related to RoPE are in the left, to Multi-node Prediction in the middle, and to Temporal Correction in the right.}
  \label{fig:ablation}
\end{figure*}

\begin{figure*}[tbh!]
  \centering
  \includegraphics[width=0.99\textwidth]{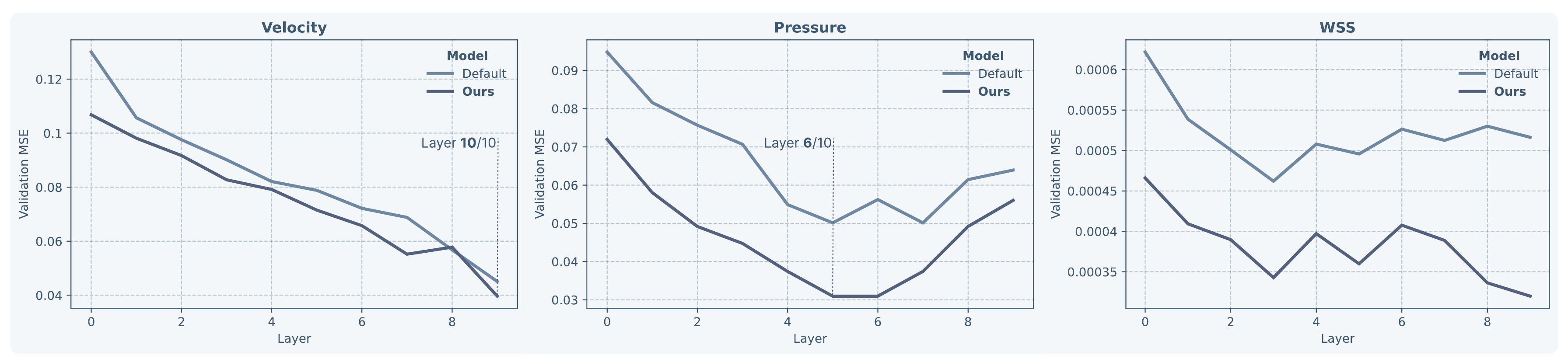}
  \caption{\textbf{Subtasks prediction.}
  We compute a regression task on next-step fields: velocity, pressure, and WSS, using latent representations from different layers. Importantly, the \emph{Default} architecture is not always the same per ablation. For \crblue{example}, the default architecture for Multi Node Prediction is an architecture using RoPE.}
  \label{fig:finetuning}
\end{figure*}

\subsection{Other Experiments}

\paragraph{Latent Representation} To study the impact of Multi-Node Prediction on Latent Representation, we conduct a simple experiment. During training, after each step of a model, we compute its $L_2$ difference with the latent target (by latent target, we mean the encoded version of the next-step target). In \autoref{fig:latent}, we present the results on a message-passing GNN with and without Multi-Node Prediction. We obtain similar results on all datasets and all architectures. First, we can see that this approach offers a latent representation that is much closer to the target. Second, while not optimizing it at all, we can see that this loss remains stable or even decreases, while it increases without Multi-Node Prediction. We believe this is a key explanation as to why this new supervised task offers such improvements on long-term rollout.

\paragraph{Predicting subtasks} We used the latent representation of each layer (layer 0 being post-Encoder) of a Transformer with 10 layers, with and without Multi-Node Prediction training. We then fit a simple 2-layer MLP to predict different quantities such as the next step Velocity, Pressure and Wall Shear Stress (WSS) \autoref{fig:finetuning}. It's important to note that while the current step velocity is present in the inputs, that is not the case for the pressure and the WSS. We first find that the latent representation obtained using Multi-Node Prediction offers lower loss. This will be constant throughout all experiments. While predicting the pressure, similarly to some tasks in computer vision, we find that the errors decrease before increasing again, a sweet spot being around 2/3 of the spatial computation. Finally, since it is extremely related to the velocity field and its differential outcomes, WSS prediction decreases with the amount of spatial processing. Interestingly, a model trained without Multi-Node Prediction doesn't see such a decrease. 

\paragraph{\crblue{Aneurysm physical metrics}} \hypertarget{upd:physical-metrics}{}\crblue{Because RMSE alone does not guarantee physical validity, we also evaluated spatially averaged WSS over all timesteps on three Aneurysm test cases and mass flow rate over time at four 2D planes inside the main artery. These supplementary plots show stable physical metrics throughout autoregressive rollout, even though the final model is not trained with a physics-informed loss.}

\section{Conclusion}

In this study, we presented 3 improvements for standard approaches to mesh-based simulations. We demonstrated that those approaches increase performance on both short-term and long-term prediction, while improving the latent representation of the graph throughout training. Importantly, our methods are architecture and dataset-agnostic \crblue{across the tested settings}.


\section*{Impact Statement}

This paper presents work whose goal is to advance the field of 
Machine Learning. There are many potential societal consequences 
of our work, none which we feel must be specifically highlighted here.

\section*{\crblue{Code availability}}
\hypertarget{upd:code}{}\crblue{Code, configuration files, and evaluation scripts will be released with the camera-ready supplementary material or the associated public repository.}

\bibliography{example_paper,added_camera_ready_refs}

@inproceedings{lienen2022finiteelementnetworks,
  title = {Learning the Dynamics of Physical Systems from Sparse Observations with Finite Element Networks},
  author = {Lienen, Marten and G{\"u}nnemann, Stephan},
  booktitle = {International Conference on Learning Representations},
  year = {2022}
}

@inproceedings{yu2025piorf,
  title = {{PIORF}: Physics-Informed Ollivier--Ricci Flow for Long-Range Interactions in Mesh Graph Neural Networks},
  author = {Yu, Youn-Yeol and Choi, Jeongwhan and Park, Jaehyeon and Lee, Kookjin and Park, Noseong},
  booktitle = {International Conference on Learning Representations},
  year = {2025}
}

@inproceedings{hoang2026hamiltonian,
  title = {Improving Long-Range Interactions in Graph Neural Simulators via Hamiltonian Dynamics},
  author = {Hoang, Tai and Trenta, Alessandro and Gravina, Alessio and Freymuth, Niklas and Becker, Philipp and Bacciu, Davide and Neumann, Gerhard},
  booktitle = {International Conference on Learning Representations},
  year = {2026}
}

@inproceedings{yu2024hcmt,
  title = {Learning Flexible Body Collision Dynamics with Hierarchical Contact Mesh Transformer},
  author = {Yu, Youn-Yeol and Choi, Jeongwhan and Cho, Woojin and Lee, Kookjin and Kim, Nayong and Chang, Kiseok and Woo, ChangSeung and Kim, Ilho and Lee, SeokWoo and Yang, Joon Young and Yoon, Sooyoung and Park, Noseong},
  booktitle = {International Conference on Learning Representations},
  year = {2024}
}

@inproceedings{janny2023eagle,
  title = {{EAGLE}: Large-scale Learning of Turbulent Fluid Dynamics with Mesh Transformers},
  author = {Janny, Steeven and B{\'e}n{\'e}teau, Aur{\'e}lien and Nadri, Madiha and Digne, Julie and Thome, Nicolas and Wolf, Christian},
  booktitle = {International Conference on Learning Representations},
  year = {2023}
}

@misc{vaswani2023attention,
      title={Attention Is All You Need}, 
      author={Ashish Vaswani and Noam Shazeer and Niki Parmar and Jakob Uszkoreit and Llion Jones and Aidan N. Gomez and Lukasz Kaiser and Illia Polosukhin},
      year={2017},
      eprint={1706.03762},
      archivePrefix={arXiv},
      primaryClass={cs.CL}
}

@misc{pfaff2021learning,
      title={Learning Mesh-Based Simulation with Graph Networks}, 
      author={Tobias Pfaff and Meire Fortunato and Alvaro Sanchez-Gonzalez and Peter W. Battaglia},
      year={2021},
      eprint={2010.03409},
      archivePrefix={arXiv},
      primaryClass={cs.LG}
}

@misc{sanchezgonzalez2020learning,
      title={Learning to Simulate Complex Physics with Graph Networks}, 
      author={Alvaro Sanchez-Gonzalez and Jonathan Godwin and Tobias Pfaff and Rex Ying and Jure Leskovec and Peter W. Battaglia},
      year={2020},
      eprint={2002.09405},
      archivePrefix={arXiv},
      primaryClass={cs.LG}
}

@ARTICLE{aneurysmdataset,

AUTHOR={Goetz, Aurèle  and Jeken-Rico, Pablo  and Pelissier, Ugo  and Chau, Yves  and Sédat, Jacques  and Hachem, Elie },

TITLE={AnXplore: a comprehensive fluid-structure interaction study of 101 intracranial aneurysms},

JOURNAL={Frontiers in Bioengineering and Biotechnology},

VOLUME={12},

YEAR={2024},

URL={https://www.frontiersin.org/journals/bioengineering-and-biotechnology/articles/10.3389/fbioe.2024.1433811},

DOI={10.3389/fbioe.2024.1433811},

ISSN={2296-4185}}

@misc{kreuzer2021rethinkinggraphtransformersspectral,
      title={Rethinking Graph Transformers with Spectral Attention}, 
      author={Devin Kreuzer and Dominique Beaini and William L. Hamilton and Vincent Létourneau and Prudencio Tossou},
      year={2021},
      eprint={2106.03893},
      archivePrefix={arXiv},
      primaryClass={cs.LG},
      url={https://arxiv.org/abs/2106.03893}, 
}

@misc{dwivedi2021generalizationtransformernetworksgraphs,
      title={A Generalization of Transformer Networks to Graphs}, 
      author={Vijay Prakash Dwivedi and Xavier Bresson},
      year={2021},
      eprint={2012.09699},
      archivePrefix={arXiv},
      primaryClass={cs.LG},
      url={https://arxiv.org/abs/2012.09699}, 
}

@article{Thuerey2018,
  author       = {Nils Thuerey and
                  Konstantin Weissenow and
                  Harshit Mehrotra and
                  Nischal Mainali and
                  Lukas Prantl and
                  Xiangyu Hu},
  title        = {Well, how accurate is it? {A} Study of Deep Learning Methods for Reynolds-Averaged
                  Navier-Stokes Simulations},
  journal      = {CoRR},
  volume       = {abs/1810.08217},
  year         = {2018},
  url          = {http://arxiv.org/abs/1810.08217},
  eprinttype    = {arXiv},
  eprint       = {1810.08217},
  bibsource    = {dblp computer science bibliography, https://dblp.org}
}

@article{Battaglia2018,
  author       = {Peter W. Battaglia and
                  Jessica B. Hamrick and
                  Victor Bapst and
                  Alvaro Sanchez{-}Gonzalez and
                  Vin{\'{\i}}cius Flores Zambaldi and
                  Mateusz Malinowski and
                  Andrea Tacchetti and
                  David Raposo and
                  Adam Santoro and
                  Ryan Faulkner and
                  {\c{C}}aglar G{\"{u}}l{\c{c}}ehre and
                  H. Francis Song and
                  Andrew J. Ballard and
                  Justin Gilmer and
                  George E. Dahl and
                  Ashish Vaswani and
                  Kelsey R. Allen and
                  Charles Nash and
                  Victoria Langston and
                  Chris Dyer and
                  Nicolas Heess and
                  Daan Wierstra and
                  Pushmeet Kohli and
                  Matthew M. Botvinick and
                  Oriol Vinyals and
                  Yujia Li and
                  Razvan Pascanu},
  title        = {Relational inductive biases, deep learning, and graph networks},
  journal      = {CoRR},
  volume       = {abs/1806.01261},
  year         = {2018},
  url          = {http://arxiv.org/abs/1806.01261},
  eprinttype    = {arXiv},
  eprint       = {1806.01261},
  bibsource    = {dblp computer science bibliography, https://dblp.org}
}

@ARTICLE{Fortunato2022,
       author = {{Fortunato}, Meire and {Pfaff}, Tobias and {Wirnsberger}, Peter and {Pritzel}, Alexander and {Battaglia}, Peter},
        title = "{MultiScale MeshGraphNets}",
      journal = {arXiv e-prints},
         year = 2022,
        month = oct,
          eid = {arXiv:2210.00612},
        pages = {arXiv:2210.00612},
          doi = {10.48550/arXiv.2210.00612},
archivePrefix = {arXiv},
       eprint = {2210.00612},
 primaryClass = {cs.LG},
       adsurl = {https://ui.adsabs.harvard.edu/abs/2022arXiv221000612F}
}

@online{comsol,
  author = {Comsol multiphysics®},
  title = {Comsol},
  year = 2020,
  url = {http://comsol.com}
}

@article{cimlib,
author = {Digonnet, Hugues and Silva, Luisa and Coupez, Thierry},
year = {2007},
month = {05},
pages = {269-274},
title = {Cimlib: A Fully Parallel Application For Numerical Simulations Based On Components Assembly},
volume = {908},
journal = {AIP Conference Proceedings},
doi = {10.1063/1.2740823}
}

@ARTICLE{gnn,
  author={Scarselli, Franco and Gori, Marco and Tsoi, Ah Chung and Hagenbuchner, Markus and Monfardini, Gabriele},
  journal={IEEE Transactions on Neural Networks}, 
  title={The Graph Neural Network Model}, 
  year={2009},
  volume={20},
  number={1},
  pages={61-80},
  doi={10.1109/TNN.2008.2005605}}

@article{Suk_2024,
   title={Mesh neural networks for SE(3)-equivariant hemodynamics estimation on the artery wall},
   volume={173},
   ISSN={0010-4825},
   url={http://dx.doi.org/10.1016/j.compbiomed.2024.108328},
   DOI={10.1016/j.compbiomed.2024.108328},
   journal={Computers in Biology and Medicine},
   publisher={Elsevier BV},
   author={Suk, Julian and de Haan, Pim and Lippe, Phillip and Brune, Christoph and Wolterink, Jelmer M.},
   year={2024},
   month=may, pages={108328} }

@misc{dwivedi2022graphneuralnetworkslearnable,
      title={Graph Neural Networks with Learnable Structural and Positional Representations}, 
      author={Vijay Prakash Dwivedi and Anh Tuan Luu and Thomas Laurent and Yoshua Bengio and Xavier Bresson},
      year={2022},
      eprint={2110.07875},
      archivePrefix={arXiv},
      primaryClass={cs.LG},
      url={https://arxiv.org/abs/2110.07875}, 
}

@misc{dauphin2017language,
      title={Language Modeling with Gated Convolutional Networks}, 
      author={Yann N. Dauphin and Angela Fan and Michael Auli and David Grangier},
      year={2017},
      eprint={1612.08083},
      archivePrefix={arXiv},
      primaryClass={cs.CL}
}

@misc{loshchilov2019decoupledweightdecayregularization,
      title={Decoupled Weight Decay Regularization}, 
      author={Ilya Loshchilov and Frank Hutter},
      year={2019},
      eprint={1711.05101},
      archivePrefix={arXiv},
      primaryClass={cs.LG},
      url={https://arxiv.org/abs/1711.05101}, 
}

@misc{garnier2024multigridgraphneuralnetworks,
      title={Multi-Grid Graph Neural Networks with Self-Attention for Computational Mechanics}, 
      author={Paul Garnier and Jonathan Viquerat and Elie Hachem},
      year={2024},
      eprint={2409.11899},
      archivePrefix={arXiv},
      primaryClass={cs.LG},
      url={https://arxiv.org/abs/2409.11899}, 
}

@misc{garnier2025trainingtransformersmeshbasedsimulations,
      title={Training Transformers for Mesh-Based Simulations}, 
      author={Paul Garnier and Vincent Lannelongue and Jonathan Viquerat and Elie Hachem},
      year={2025},
      eprint={2508.18051},
      archivePrefix={arXiv},
      primaryClass={cs.LG},
      url={https://arxiv.org/abs/2508.18051}, 
}

@misc{garnier2024meshmask,
      title={MeshMask: Physics-Based Simulations with Masked Graph Neural Networks}, 
      author={Paul Garnier and Vincent Lannelongue and Jonathan Viquerat and Elie Hachem},
      year={2025},
      eprint={2501.08738},
      archivePrefix={arXiv},
      primaryClass={cs.LG},
      url={https://arxiv.org/abs/2501.08738}, 
}

@article{RAISSI2019686,
title = {Physics-informed neural networks: A deep learning framework for solving forward and inverse problems involving nonlinear partial differential equations},
journal = {Journal of Computational Physics},
volume = {378},
pages = {686-707},
year = {2019},
issn = {0021-9991},
doi = {https://doi.org/10.1016/j.jcp.2018.10.045},
url = {https://www.sciencedirect.com/science/article/pii/S0021999118307125},
author = {M. Raissi and P. Perdikaris and G.E. Karniadakis}
}

@misc{lam2023graphcastlearningskillfulmediumrange,
      title={GraphCast: Learning skillful medium-range global weather forecasting}, 
      author={Remi Lam and Alvaro Sanchez-Gonzalez and Matthew Willson and Peter Wirnsberger and Meire Fortunato and Ferran Alet and Suman Ravuri and Timo Ewalds and Zach Eaton-Rosen and Weihua Hu and Alexander Merose and Stephan Hoyer and George Holland and Oriol Vinyals and Jacklynn Stott and Alexander Pritzel and Shakir Mohamed and Peter Battaglia},
      year={2023},
      eprint={2212.12794},
      archivePrefix={arXiv},
      primaryClass={cs.LG},
      url={https://arxiv.org/abs/2212.12794}, 
}

@misc{wu2024transolverfasttransformersolver,
      title={Transolver: A Fast Transformer Solver for PDEs on General Geometries}, 
      author={Haixu Wu and Huakun Luo and Haowen Wang and Jianmin Wang and Mingsheng Long},
      year={2024},
      eprint={2402.02366},
      archivePrefix={arXiv},
      primaryClass={cs.LG},
      url={https://arxiv.org/abs/2402.02366}, 
}

@inproceedings{digiovanni2023oversquashing,
  title={On over-squashing in message passing neural networks: The impact of width, depth, and topology},
  author={Di Giovanni, Francesco and Giusti, Lorenzo and Barbero, Federico and Luise, Giulia and Lio, Pietro and Bronstein, Michael M},
  booktitle={International conference on machine learning},
  pages={7865--7885},
  year={2023},
  organization={PMLR}
}

@article{Gragg1964,
  title = {Generalized Multistep Predictor-Corrector Methods},
  volume = {11},
  ISSN = {1557-735X},
  url = {http://dx.doi.org/10.1145/321217.321223},
  DOI = {10.1145/321217.321223},
  number = {2},
  journal = {Journal of the ACM},
  publisher = {Association for Computing Machinery (ACM)},
  author = {Gragg,  William B. and Stetter,  Hans J.},
  year = {1964},
  month = apr,
  pages = {188–209}
}

@article{liu2024deepseek,
  title={Deepseek-v2: A strong, economical, and efficient mixture-of-experts language model},
  author={Liu, Aixin and Feng, Bei and Wang, Bin and Wang, Bingxuan and Liu, Bo and Zhao, Chenggang and Dengr, Chengqi and Ruan, Chong and Dai, Damai and Guo, Daya and others},
  journal={arXiv preprint arXiv:2405.04434},
  year={2024}
}

@misc{kimiteam2025kimilinearexpressiveefficient,
      title={Kimi Linear: An Expressive, Efficient Attention Architecture}, 
      author={KimiTeam and Yu Zhang and Zongyu Lin and Xingcheng Yao and Jiaxi Hu and Fanqing Meng and Chengyin Liu and Xin Men and Songlin Yang and Zhiyuan Li and Wentao Li and Enzhe Lu and Weizhou Liu and Yanru Chen and Weixin Xu and Longhui Yu and Yejie Wang and Yu Fan and Longguang Zhong and Enming Yuan and Dehao Zhang and Yizhi Zhang and T. Y. Liu and Haiming Wang and Shengjun Fang and Weiran He and Shaowei Liu and Yiwei Li and Jianlin Su and Jiezhong Qiu and Bo Pang and Junjie Yan and Zhejun Jiang and Weixiao Huang and Bohong Yin and Jiacheng You and Chu Wei and Zhengtao Wang and Chao Hong and Yutian Chen and Guanduo Chen and Yucheng Wang and Huabin Zheng and Feng Wang and Yibo Liu and Mengnan Dong and Zheng Zhang and Siyuan Pan and Wenhao Wu and Yuhao Wu and Longyu Guan and Jiawen Tao and Guohong Fu and Xinran Xu and Yuzhi Wang and Guokun Lai and Yuxin Wu and Xinyu Zhou and Zhilin Yang and Yulun Du},
      year={2025},
      eprint={2510.26692},
      archivePrefix={arXiv},
      primaryClass={cs.CL},
      url={https://arxiv.org/abs/2510.26692}, 
}

@article{Cai2021,
  title = {Physics-informed neural networks (PINNs) for fluid mechanics: a review},
  volume = {37},
  ISSN = {1614-3116},
  url = {http://dx.doi.org/10.1007/s10409-021-01148-1},
  DOI = {10.1007/s10409-021-01148-1},
  number = {12},
  journal = {Acta Mechanica Sinica},
  publisher = {Springer Science and Business Media LLC},
  author = {Cai,  Shengze and Mao,  Zhiping and Wang,  Zhicheng and Yin,  Minglang and Karniadakis,  George Em},
  year = {2021},
  month = dec,
  pages = {1727–1738}
}

@article{belkin2006manifold,
  title={Manifold regularization: A geometric framework for learning from labeled and unlabeled examples.},
  author={Belkin, Mikhail and Niyogi, Partha and Sindhwani, Vikas},
  journal={Journal of machine learning research},
  volume={7},
  number={11},
  year={2006}
}

@article{zhang2018link,
  title={Link prediction based on graph neural networks},
  author={Zhang, Muhan and Chen, Yixin},
  journal={Advances in neural information processing systems},
  volume={31},
  year={2018}
}

@article{alsentzer2020subgraph,
  title={Subgraph neural networks},
  author={Alsentzer, Emily and Finlayson, Samuel and Li, Michelle and Zitnik, Marinka},
  journal={Advances in Neural Information Processing Systems},
  volume={33},
  pages={8017--8029},
  year={2020}
}

@misc{deepseekai2025deepseekv3technicalreport,
      title={DeepSeek-V3 Technical Report}, 
      author={DeepSeek-AI and Aixin Liu and Bei Feng and Bing Xue and Bingxuan Wang and Bochao Wu and Chengda Lu and Chenggang Zhao and Chengqi Deng and Chenyu Zhang and Chong Ruan and Damai Dai and Daya Guo and Dejian Yang and Deli Chen and Dongjie Ji and Erhang Li and Fangyun Lin and Fucong Dai and Fuli Luo and Guangbo Hao and Guanting Chen and Guowei Li and H. Zhang and Han Bao and Hanwei Xu and Haocheng Wang and Haowei Zhang and Honghui Ding and Huajian Xin and Huazuo Gao and Hui Li and Hui Qu and J. L. Cai and Jian Liang and Jianzhong Guo and Jiaqi Ni and Jiashi Li and Jiawei Wang and Jin Chen and Jingchang Chen and Jingyang Yuan and Junjie Qiu and Junlong Li and Junxiao Song and Kai Dong and Kai Hu and Kaige Gao and Kang Guan and Kexin Huang and Kuai Yu and Lean Wang and Lecong Zhang and Lei Xu and Leyi Xia and Liang Zhao and Litong Wang and Liyue Zhang and Meng Li and Miaojun Wang and Mingchuan Zhang and Minghua Zhang and Minghui Tang and Mingming Li and Ning Tian and Panpan Huang and Peiyi Wang and Peng Zhang and Qiancheng Wang and Qihao Zhu and Qinyu Chen and Qiushi Du and R. J. Chen and R. L. Jin and Ruiqi Ge and Ruisong Zhang and Ruizhe Pan and Runji Wang and Runxin Xu and Ruoyu Zhang and Ruyi Chen and S. S. Li and Shanghao Lu and Shangyan Zhou and Shanhuang Chen and Shaoqing Wu and Shengfeng Ye and Shengfeng Ye and Shirong Ma and Shiyu Wang and Shuang Zhou and Shuiping Yu and Shunfeng Zhou and Shuting Pan and T. Wang and Tao Yun and Tian Pei and Tianyu Sun and W. L. Xiao and Wangding Zeng and Wanjia Zhao and Wei An and Wen Liu and Wenfeng Liang and Wenjun Gao and Wenqin Yu and Wentao Zhang and X. Q. Li and Xiangyue Jin and Xianzu Wang and Xiao Bi and Xiaodong Liu and Xiaohan Wang and Xiaojin Shen and Xiaokang Chen and Xiaokang Zhang and Xiaosha Chen and Xiaotao Nie and Xiaowen Sun and Xiaoxiang Wang and Xin Cheng and Xin Liu and Xin Xie and Xingchao Liu and Xingkai Yu and Xinnan Song and Xinxia Shan and Xinyi Zhou and Xinyu Yang and Xinyuan Li and Xuecheng Su and Xuheng Lin and Y. K. Li and Y. Q. Wang and Y. X. Wei and Y. X. Zhu and Yang Zhang and Yanhong Xu and Yanhong Xu and Yanping Huang and Yao Li and Yao Zhao and Yaofeng Sun and Yaohui Li and Yaohui Wang and Yi Yu and Yi Zheng and Yichao Zhang and Yifan Shi and Yiliang Xiong and Ying He and Ying Tang and Yishi Piao and Yisong Wang and Yixuan Tan and Yiyang Ma and Yiyuan Liu and Yongqiang Guo and Yu Wu and Yuan Ou and Yuchen Zhu and Yuduan Wang and Yue Gong and Yuheng Zou and Yujia He and Yukun Zha and Yunfan Xiong and Yunxian Ma and Yuting Yan and Yuxiang Luo and Yuxiang You and Yuxuan Liu and Yuyang Zhou and Z. F. Wu and Z. Z. Ren and Zehui Ren and Zhangli Sha and Zhe Fu and Zhean Xu and Zhen Huang and Zhen Zhang and Zhenda Xie and Zhengyan Zhang and Zhewen Hao and Zhibin Gou and Zhicheng Ma and Zhigang Yan and Zhihong Shao and Zhipeng Xu and Zhiyu Wu and Zhongyu Zhang and Zhuoshu Li and Zihui Gu and Zijia Zhu and Zijun Liu and Zilin Li and Ziwei Xie and Ziyang Song and Ziyi Gao and Zizheng Pan},
      year={2025},
      eprint={2412.19437},
      archivePrefix={arXiv},
      primaryClass={cs.CL},
      url={https://arxiv.org/abs/2412.19437}, 
}

@misc{gloeckle2024betterfasterlarge,
      title={Better \& Faster Large Language Models via Multi-token Prediction}, 
      author={Fabian Gloeckle and Badr Youbi Idrissi and Baptiste Rozière and David Lopez-Paz and Gabriel Synnaeve},
      year={2024},
      eprint={2404.19737},
      archivePrefix={arXiv},
      primaryClass={cs.CL},
      url={https://arxiv.org/abs/2404.19737}, 
}

@book{butcher2016numerical,
  title={Numerical Methods for Ordinary Differential Equations},
  author={Butcher, John C.},
  year={2016},
  publisher={Wiley}
}

@article{Runge1895,
  title = {Ueber die numerische Aufl{\"o}sung von Differentialgleichungen},
  volume = {46},
  ISSN = {1432-1807},
  url = {http://dx.doi.org/10.1007/BF01446807},
  DOI = {10.1007/bf01446807},
  number = {2},
  journal = {Mathematische Annalen},
  publisher = {Springer Science and Business Media LLC},
  author = {Runge,  C.},
  year = {1895},
  month = jun,
  pages = {167–178}
}

@misc{wu2022pointtransformerv2grouped,
      title={Point Transformer V2: Grouped Vector Attention and Partition-based Pooling}, 
      author={Xiaoyang Wu and Yixing Lao and Li Jiang and Xihui Liu and Hengshuang Zhao},
      year={2022},
      eprint={2210.05666},
      archivePrefix={arXiv},
      primaryClass={cs.CV},
      url={https://arxiv.org/abs/2210.05666}, 
}

@misc{zhao2021pointtransformer,
      title={Point Transformer}, 
      author={Hengshuang Zhao and Li Jiang and Jiaya Jia and Philip Torr and Vladlen Koltun},
      year={2021},
      eprint={2012.09164},
      archivePrefix={arXiv},
      primaryClass={cs.CV},
      url={https://arxiv.org/abs/2012.09164}, 
}

@book{kutta1901beitrag,
  title={Beitrag zur n{\"a}herungsweisen Integration totaler Differentialgleichungen},
  author={Kutta, Wilhelm},
  year={1901},
  publisher={Teubner}
}

@misc{wu2021rethinkingimprovingrelativeposition,
      title={Rethinking and Improving Relative Position Encoding for Vision Transformer}, 
      author={Kan Wu and Houwen Peng and Minghao Chen and Jianlong Fu and Hongyang Chao},
      year={2021},
      eprint={2107.14222},
      archivePrefix={arXiv},
      primaryClass={cs.CV},
      url={https://arxiv.org/abs/2107.14222}, 
}

@misc{shaw2018selfattentionrelativepositionrepresentations,
      title={Self-Attention with Relative Position Representations}, 
      author={Peter Shaw and Jakob Uszkoreit and Ashish Vaswani},
      year={2018},
      eprint={1803.02155},
      archivePrefix={arXiv},
      primaryClass={cs.CL},
      url={https://arxiv.org/abs/1803.02155}, 
}

@misc{dosovitskiy2021imageworth16x16words,
      title={An Image is Worth 16x16 Words: Transformers for Image Recognition at Scale}, 
      author={Alexey Dosovitskiy and Lucas Beyer and Alexander Kolesnikov and Dirk Weissenborn and Xiaohua Zhai and Thomas Unterthiner and Mostafa Dehghani and Matthias Minderer and Georg Heigold and Sylvain Gelly and Jakob Uszkoreit and Neil Houlsby},
      year={2021},
      eprint={2010.11929},
      archivePrefix={arXiv},
      primaryClass={cs.CV},
      url={https://arxiv.org/abs/2010.11929}, 
}

@misc{alkin2025abuptscalingneuralcfd,
      title={AB-UPT: Scaling Neural CFD Surrogates for High-Fidelity Automotive Aerodynamics Simulations via Anchored-Branched Universal Physics Transformers}, 
      author={Benedikt Alkin and Maurits Bleeker and Richard Kurle and Tobias Kronlachner and Reinhard Sonnleitner and Matthias Dorfer and Johannes Brandstetter},
      year={2025},
      eprint={2502.09692},
      archivePrefix={arXiv},
      primaryClass={cs.LG},
      url={https://arxiv.org/abs/2502.09692}, 
}

@misc{qiu2025gatedattentionlargelanguage,
      title={Gated Attention for Large Language Models: Non-linearity, Sparsity, and Attention-Sink-Free}, 
      author={Zihan Qiu and Zekun Wang and Bo Zheng and Zeyu Huang and Kaiyue Wen and Songlin Yang and Rui Men and Le Yu and Fei Huang and Suozhi Huang and Dayiheng Liu and Jingren Zhou and Junyang Lin},
      year={2025},
      eprint={2505.06708},
      archivePrefix={arXiv},
      primaryClass={cs.CL},
      url={https://arxiv.org/abs/2505.06708}, 
}

@article{Liu2025,
  title = {Automatic network structure discovery of physics informed neural networks via knowledge distillation},
  volume = {16},
  ISSN = {2041-1723},
  url = {http://dx.doi.org/10.1038/s41467-025-64624-3},
  DOI = {10.1038/s41467-025-64624-3},
  number = {1},
  journal = {Nature Communications},
  publisher = {Springer Science and Business Media LLC},
  author = {Liu,  Ziti and Liu,  Yang and Yan,  Xunshi and Liu,  Wen and Nie,  Han and Guo,  Shuaiqi and Zhang,  Chen-an},
  year = {2025},
  month = oct 
}

@book{hairer1993solving,
  title={Solving Ordinary Differential Equations I: Nonstiff Problems},
  author={Hairer, Ernst and N{\o}rsett, Syvert P. and Wanner, Gerhard},
  volume={8},
  series={Springer Series in Computational Mathematics},
  year={1993},
  publisher={Springer}
}
\bibliographystyle{icml2026}

\newpage
\appendix
\onecolumn
\section{Datasets}
\label{appendix:datasets}

\begin{figure}[!t]
  \centering
  \includegraphics[width=0.99\textwidth]{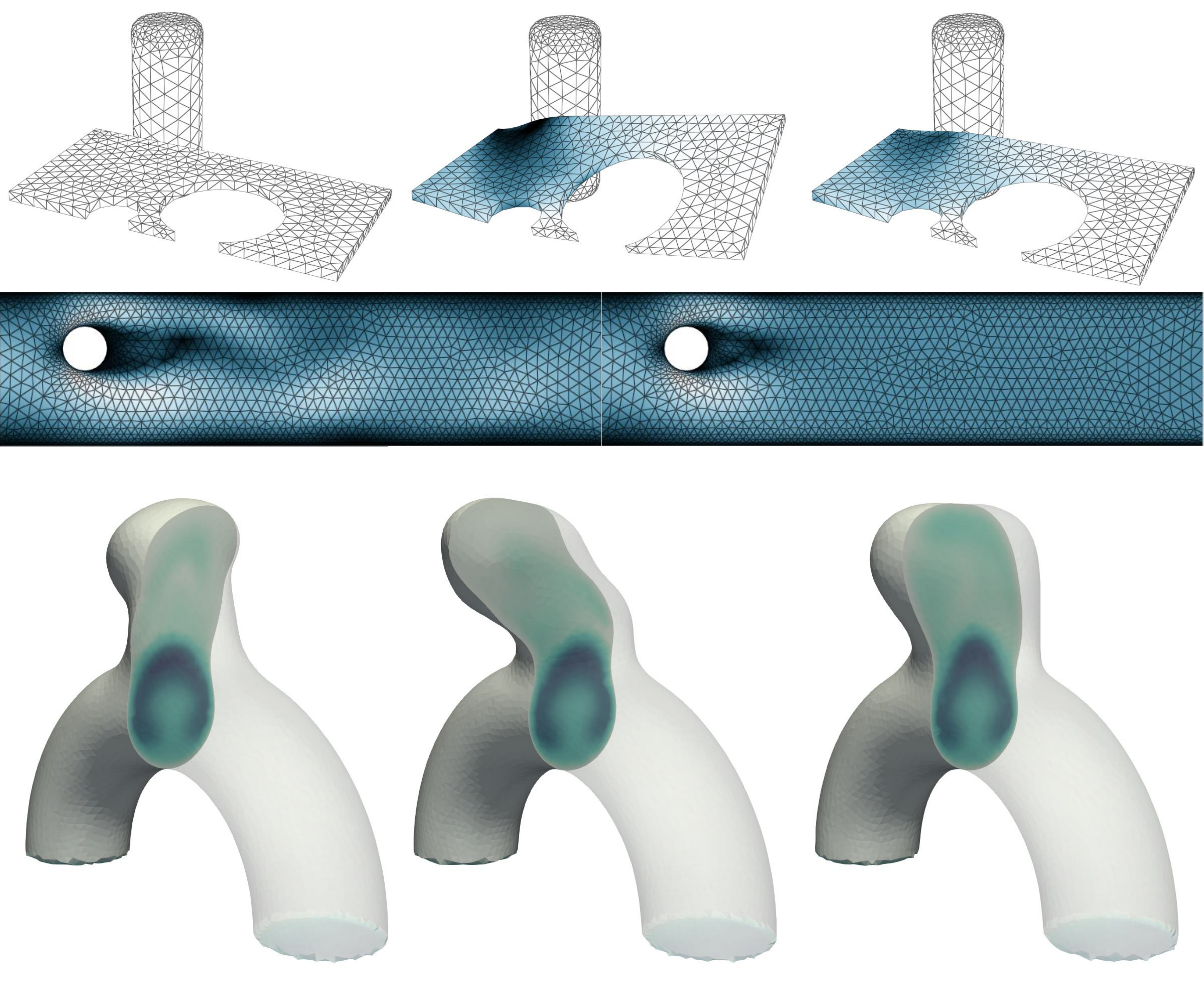}
\caption{
We display details of our three \crblue{primary} datasets: DeformingPlate, Cylinder and \crblue{Coarse} Aneurysm. For the first two, we also display the mesh used for the simulations, while we discard it from the Aneurysm visualisation for practicality purposes.
} 
  \label{fig:datasetsdetails}
\end{figure}

We give details below about the inputs and outputs used for each dataset (see Table \ref{tab:datasets-details} and \autoref{fig:datasetsdetails}). \dataset{Cylinder} and DeformingPlate were generated with COMSOL \cite{comsol} and were introduced by \cite{pfaff2021learning}. \dataset{Aneurysm} was generated with CimLib \cite{cimlib} and was introduced by \cite{aneurysmdataset,garnier2024meshmask}. 

\renewcommand{\arraystretch}{1.1}
\begin{table}[t]
\centering
\begin{small}
\caption{Datasets used in the experiments, their features as well as if some temporal history is used. Finally, we also present the average number of nodes per mesh, \crblue{together with solver, train/test trajectories, rollout length, and time step}.}
\label{tab:datasets-details}
\hypertarget{upd:dataset-table}{}\resizebox{\textwidth}{!}{%
\begin{tabular}{|p{28mm}||c|c|c|c|c|c|c|c|c|}
	\hline
	\textbf{Dataset} &  
	\textbf{Inputs} &
	\textbf{Outputs} &
	\textbf{History} &
	\textbf{Nodes} &
    \crblue{\textbf{Solver}} &
    \crblue{\textbf{Train}} &
    \crblue{\textbf{Test}} &
    \crblue{\textbf{Steps}} &
    \crblue{$\boldsymbol{\Delta t}$} \\
	\hline\hline
    \dataset{Cylinder} & $n, v_x, v_y$ & $v_x, v_y$ & 0 & 2000 & \crblue{COMSOL} & \crblue{1000} & \crblue{100} & \crblue{600} & \crblue{0.01s} \\
    \hline
    \crblue{\dataset{DeformingPlate}} & $n, x, y, z, f_{\text{in}}$ & $x, y, z, \sigma$ & 0 & 1000 & \crblue{COMSOL} & \crblue{1000} & \crblue{100} & \crblue{400} & \crblue{--} \\
    \hline
    \crblue{\dataset{Coarse Aneurysm}} & $n, v_x, v_y, v_z, v_{\text{in}}$ & $v_x, v_y, v_z$ & 1 & \crblue{20000} & \crblue{CimLib} & \crblue{100} & \crblue{20} & \crblue{80} & \crblue{0.01s} \\
    \hline
    \crblue{\dataset{Fine Aneurysm}} & \crblue{$n, v_x, v_y, v_z, v_{\text{in}}$} & \crblue{$v_x, v_y, v_z$} & \crblue{1} & \crblue{250000} & \crblue{CimLib} & \crblue{100} & \crblue{20} & \crblue{80} & \crblue{0.01s} \\
    \hline
\end{tabular}}
\end{small}
\end{table}
\renewcommand{\arraystretch}{1.0}

In Table \ref{tab:datasets-details}, $n$ is the node type (Inflow, Outflow, Wall, Obstacle, Normal) and $v_{\text{in}}$ the inflow velocity at the current timestep. \crblue{Boundary values are enforced at prediction time using these node types: velocity is enforced for all node types except Normal and Outflow, and pressure is enforced only for Inlet nodes.} When history is different than $0$, we use a first-order derivative of the inputs as an extra feature. For example, we add $a_x, a_y, a_z$ to each node from an aneurysm mesh.

\subsection{Noise}

We display the noise used as well as the global nodes selected for each dataset in Table \ref{tab:noises}.

\paragraph{Noise Scale} We make our inputs noisy by following the same strategy as \cite{sanchezgonzalez2020learning}.
We add random noise $\mathcal{N}(0,\sigma)$ to the dynamical inputs. Each noise standard deviation was selected using the following procedure. We train one model without any noise and then compute the distribution of the one-step error. We use this error deviation as the noise scale. For the case of the aneurysm dataset, the variational nature of the inflow makes for said distribution to be time-dependant. While we investigated different strategies, such as time-dependant noise, we simply used the larger standard deviation possible($\max\limits_{t \in [1, T-1]}\sigma _t$). 

\renewcommand{\arraystretch}{1.1}
\begin{table}[t]
\centering
\begin{small}
\caption{Noise levels applied to each dataset, per feature, during training.}
\label{tab:noises}
\begin{tabular}{|p{35.5mm}||c|}
	\hline
	\textbf{Dataset} & \textbf{Noise} \\
	\hline\hline
    \dataset{Cylinder} & 0.02 \\
    \hline
    \dataset{Plate} & 0.003 \\
    \hline
    \dataset{3D-Aneurysm} & $v_x, v_y$: 10,\; $v_z$: 1 \\
    \hline
\end{tabular}
\end{small}
\end{table}
\renewcommand{\arraystretch}{1.0}

\section{Theoretical Analysis}
\label{sec:theoretical_analysis}

We provide theoretical grounding for our architectural contributions by connecting them to two classical themes in numerical PDEs:
(i) \emph{spatial consistency}, where stability and accuracy depend on controlling derivative (flux) errors rather than only pointwise errors; and
(ii) \emph{time-integration stability}, where explicit one-step methods are prone to drift for stiff semi-discrete systems.
Throughout, we view the mesh/graph as nodes $\{ \mathbf{x}_i \}_{i=1}^N \subset \Omega \subset \mathbb{R}^d$ with edges given by mesh adjacency.
We denote the (maximum) local mesh size by
$h := \max_{(i,j)\in E} \| \mathbf{x}_j - \mathbf{x}_i \|$ and the neighborhood of node $i$ by $\mathcal{N}(i)$.
For simplicity the analysis is stated for a scalar field $u:\Omega\to\mathbb{R}$; the vector-valued case $u:\Omega\to\mathbb{R}^c$ follows componentwise with identical bounds.

\subsection{Multi Node Prediction as Sobolev Regularization}
\label{subsec:theory_mnp}

Standard surrogates usually minimize a node-wise $L^2$ error:
\[
\mathcal{L}_{\mathrm{node}}
:= \frac{1}{N} \sum_{i=1}^N \big\| u(\mathbf{x}_i) - \hat{u}_i \big\|^2
\]
which controls pointwise accuracy but does not directly control errors in spatial derivatives.
Yet, many PDE operators (in strong or weak form) depend on gradients and fluxes, e.g.\ diffusion energies $\int_\Omega \|\nabla u\|^2$ or nonlinear fluxes $f(u,\nabla u)$.
Consequently, purely node-wise supervision can allow locally inconsistent stencils (large derivative/flux errors) even when pointwise errors are small, a behavior that is particularly costly in long rollouts. MNP augments node-wise supervision with \emph{patch} supervision: for each center node $i$, the model produces predictions for its neighborhood:
\[
\hat{u}_{j|i} \approx u(\mathbf{x}_j) \qquad j\in \mathcal{N}(i)
\]
In our implementation (Sec.~\ref{subsec:mnp}), these predictions are produced by a one-layer star transformer that processes the center latent together with freshly encoded neighbor features. Formally, one can write:
\[
(\mathbf{o}_{k|i})_{k\in\{i\}\cup\mathcal{N}(i)}
=
T_{\theta'}\!\Big([\mathbf{z}^L_i;\ \{\mathbf{z}^0_j\}_{j\in\mathcal{N}(i)}]\Big)
\qquad
\hat{u}_{j|i}=\mathcal{D}_{\theta}(\mathbf{o}_{j|i})
\]
but the analysis below is \emph{architecture-agnostic}: it only depends on the resulting scalar predictions $\hat u_{j|i}$ and therefore applies equally to (i) direct decoding from $\mathbf{z}^L_i$ and (ii) the implemented auxiliary-transformer variant.

We define the MNP loss as the following:
\begin{equation}
\mathcal{L}_{\mathrm{MNP}}
\;:=\;
\frac{1}{N}\sum_{i=1}^N \sum_{j\in\mathcal{N}(i)} w_{ij}\,\big\| \hat{u}_{j|i} - u(\mathbf{x}_j) \big\|^2
\label{eq:mnp_loss}
\end{equation}
for nonnegative weights $w_{ij}$. Importantly, the discrete gradients used below depend on differences $(v_j-v_i)$ and therefore necessarily involve the \emph{center} error $|\hat u_i-u(\mathbf{x}_i)|$. In our training setup this term is provided by $\mathcal{L}_{\mathrm{node}}$ (i.e., $\mathcal{L}_{\mathrm{main}}$ in Sec.~\ref{subsec:mnp}).

Given values $\{v_j\}_{j\in\{i\}\cup\mathcal{N}(i)}$, define the \emph{local weighted least-squares} (WLS) discrete gradient $\nabla_h v(\mathbf{x}_i)\in\mathbb{R}^d$ as:
\begin{equation}
\nabla_h v(\mathbf{x}_i)
\;:=\;
\arg\min_{\mathbf{g}\in\mathbb{R}^d}
\sum_{j\in\mathcal{N}(i)} w_{ij}\,\big| (v_j-v_i) - \mathbf{g}^\top(\mathbf{x}_j-\mathbf{x}_i) \big|^2
\label{eq:lsq_grad_def}
\end{equation}
Let $\mathbf{B}_i\in\mathbb{R}^{|\mathcal{N}(i)|\times d}$ have rows $(\mathbf{x}_j-\mathbf{x}_i)^\top$ and $\mathbf{W}_i:=\mathrm{diag}(w_{ij})$.
With $\mathbf{d}_i(v):=(v_j-v_i)_{j\in\mathcal{N}(i)}$, the minimizer has the closed form
\begin{equation}
\nabla_h v(\mathbf{x}_i)
=
\Big( \mathbf{B}_i^\top \mathbf{W}_i \mathbf{B}_i \Big)^{-1}\mathbf{B}_i^\top \mathbf{W}_i\,\mathbf{d}_i(v)
\label{eq:lsq_grad_closed}
\end{equation}
whenever $\mathbf{B}_i^\top \mathbf{W}_i \mathbf{B}_i$ is invertible.

\begin{assumption}[]
\label{ass:wls_regular}
For every node $i$, the weights satisfy $w_{ij}\ge 0$ and are normalized as $\sum_{j\in\mathcal{N}(i)} w_{ij}=1$.
Moreover, there exist constants $0<c_0\le c_1<\infty$, independent of $h$ and $i$, such that the weighted moment matrix obeys the uniform spectral bounds
\begin{equation}
c_0\,h^2\,\mathbf{I}_d
\;\preceq\;
\mathbf{B}_i^\top \mathbf{W}_i \mathbf{B}_i
\;\preceq\;
c_1\,h^2\,\mathbf{I}_d
\label{eq:patch_nondegeneracy}
\end{equation}
These conditions hold for quasi-uniform/shape-regular meshes with bounded-degree neighborhoods and standard distance-based weights normalized to unit sum, provided each neighborhood spans $\mathbb{R}^d$ with a uniform angle condition.
\end{assumption}

\begin{theorem}[\textbf{MNP + node-wise supervision controls a discrete $H^1$ seminorm}]
\label{thm:sobolev}
Assume $u\in C^2(\Omega)$ and Assumption~\ref{ass:wls_regular} holds.
Define the \emph{predicted patch} at node $i$ by setting $v_i=\hat u_i$ (the standard node-wise prediction) and $v_j=\hat u_{j|i}$ for $j\in\mathcal{N}(i)$, and let
$\widehat{\nabla}_h u(\mathbf{x}_i):=\nabla_h v(\mathbf{x}_i)$ be the WLS gradient \eqref{eq:lsq_grad_def} computed from these predicted values.
Then there exists a constant $C>0$, independent of $h$, such that for every node $i$,
\begin{equation}
\big\| \widehat{\nabla}_h u(\mathbf{x}_i) - \nabla u(\mathbf{x}_i) \big\|^2
\;\le\;
\frac{C}{h^2}\!\left(
\sum_{j\in\mathcal{N}(i)} w_{ij}\,\big|\hat{u}_{j|i} - u(\mathbf{x}_j)\big|^2
+
\big|\hat u_i - u(\mathbf{x}_i)\big|^2
\right)
\;+\;
C\,h^2\,\|u\|_{C^2(\Omega)}^2
\label{eq:h1_bound_local}
\end{equation}
Consequently,
\begin{equation}
\frac{1}{N}\sum_{i=1}^N
\big\| \widehat{\nabla}_h u(\mathbf{x}_i) - \nabla u(\mathbf{x}_i) \big\|^2
\;\le\;
\frac{C}{h^2}\Big(\mathcal{L}_{\mathrm{MNP}}+\mathcal{L}_{\mathrm{node}}\Big)
\;+\;
C\,h^2\,\|u\|_{C^2(\Omega)}^2
\label{eq:h1_bound_global}
\end{equation}
The same result holds componentwise for vector-valued fields $u:\Omega\to\mathbb{R}^c$.
\end{theorem}

\begin{proof}
Fix a node $i$ and denote $\mathbf{d}_{ij}:=\mathbf{x}_j-\mathbf{x}_i$.
Let $v$ and $\tilde v$ be two sets of patch values on $\{i\}\cup\mathcal{N}(i)$.
From \eqref{eq:lsq_grad_closed}, we have
\[
\nabla_h v(\mathbf{x}_i)-\nabla_h \tilde v(\mathbf{x}_i)
=
\Big(\mathbf{B}_i^\top \mathbf{W}_i\mathbf{B}_i\Big)^{-1}\mathbf{B}_i^\top \mathbf{W}_i\,\Big(\mathbf{d}_i(v)-\mathbf{d}_i(\tilde v)\Big)
\]
Taking norms and using $\|\mathbf{B}_i^\top \mathbf{W}_i\mathbf{q}\|
\le \|\mathbf{B}_i^\top \mathbf{W}_i^{1/2}\|\,\|\mathbf{q}\|_{\mathbf{W}_i}$ yields
\[
\big\|\nabla_h v(\mathbf{x}_i)-\nabla_h \tilde v(\mathbf{x}_i)\big\|
\le
\big\|\big(\mathbf{B}_i^\top \mathbf{W}_i\mathbf{B}_i\big)^{-1}\big\|\,
\big\|\mathbf{B}_i^\top \mathbf{W}_i^{1/2}\big\|\,
\big\|\mathbf{d}_i(v)-\mathbf{d}_i(\tilde v)\big\|_{\mathbf{W}_i}
\]
By Assumption~\ref{ass:wls_regular},
$\|(\mathbf{B}_i^\top \mathbf{W}_i\mathbf{B}_i)^{-1}\|\le (c_0 h^2)^{-1}$ and
$\|\mathbf{B}_i^\top \mathbf{W}_i^{1/2}\|^2=\lambda_{\max}(\mathbf{B}_i^\top \mathbf{W}_i\mathbf{B}_i)\le c_1 h^2$,
so we ahve the following:
\begin{equation}
\big\|\nabla_h v(\mathbf{x}_i)-\nabla_h \tilde v(\mathbf{x}_i)\big\|
\;\le\;
\frac{\sqrt{c_1}}{c_0}\,\frac{1}{h}\,
\big\|\mathbf{d}_i(v)-\mathbf{d}_i(\tilde v)\big\|_{\mathbf{W}_i}
\label{eq:lsq_sensitivity_bound}
\end{equation}

Now set $v_i=\hat u_i$, $v_j=\hat u_{j|i}$ for $j\in\mathcal{N}(i)$, and $\tilde v_k=u(\mathbf{x}_k)$ for all $k$.
Define the errors $e_i:=\hat u_i-u(\mathbf{x}_i)$ and $e_{j|i}:=\hat u_{j|i}-u(\mathbf{x}_j)$.
Then for each $j\in\mathcal{N}(i)$:
\[
\big(\mathbf{d}_i(v)-\mathbf{d}_i(\tilde v)\big)_j
=
(\hat u_{j|i}-\hat u_i)-(u(\mathbf{x}_j)-u(\mathbf{x}_i))
=
e_{j|i}-e_i
\]
hence
\[
\big\|\mathbf{d}_i(v)-\mathbf{d}_i(\tilde v)\big\|_{\mathbf{W}_i}^2
=
\sum_{j\in\mathcal{N}(i)} w_{ij}\,|e_{j|i}-e_i|^2
\le
2\sum_{j\in\mathcal{N}(i)} w_{ij}\,|e_{j|i}|^2
+
2\Big(\sum_{j\in\mathcal{N}(i)} w_{ij}\Big)\,|e_i|^2
\]
Using the normalization $\sum_{j} w_{ij}=1$ in Assumption~\ref{ass:wls_regular} and combining with \eqref{eq:lsq_sensitivity_bound} gives us:
\begin{equation}
\big\|\nabla_h \hat u_{\cdot|i}(\mathbf{x}_i)-\nabla_h u(\mathbf{x}_i)\big\|^2
\;\le\;
\frac{C}{h^2}\!\left(
\sum_{j\in\mathcal{N}(i)} w_{ij}\,|e_{j|i}|^2
+
|e_i|^2
\right)
\label{eq:grad_diff_bound_correct}
\end{equation}
for a constant $C$ independent of $h$.

It remains to relate $\nabla_h u(\mathbf{x}_i)$ to the true gradient $\nabla u(\mathbf{x}_i)$.
By Taylor expansion, for each $j\in\mathcal{N}(i)$, we have:
\[
u(\mathbf{x}_j) - u(\mathbf{x}_i)
=
\nabla u(\mathbf{x}_i)^\top \mathbf{d}_{ij}
+
r_{ij}
\qquad
|r_{ij}|\le C_1\,\|u\|_{C^2(\Omega)}\,\|\mathbf{d}_{ij}\|^2 \le C_1\,\|u\|_{C^2(\Omega)}\,h^2
\]
Evaluating the WLS objective \eqref{eq:lsq_grad_def} at $\mathbf{g}=\nabla u(\mathbf{x}_i)$ yields residuals $r_{ij}$ and thus
\[
J_i(\nabla u(\mathbf{x}_i);u)
=
\sum_{j\in\mathcal{N}(i)} w_{ij}\,|r_{ij}|^2
\le
C_2\,h^4\,\|u\|_{C^2(\Omega)}^2
\]
where we used again $\sum_j w_{ij}=1$.
Since $\nabla_h u(\mathbf{x}_i)$ minimizes $J_i(\cdot;u)$ and $J_i(\nabla_h u(\mathbf{x}_i);u)\ge 0$, strong convexity of $J_i$ with Hessian $2\mathbf{B}_i^\top\mathbf{W}_i\mathbf{B}_i$ implies
\[
c_0\,h^2\,\big\|\nabla_h u(\mathbf{x}_i)-\nabla u(\mathbf{x}_i)\big\|^2
\le
J_i(\nabla u(\mathbf{x}_i);u)
\le
C_2\,h^4\,\|u\|_{C^2(\Omega)}^2
\]
hence
\begin{equation}
\big\|\nabla_h u(\mathbf{x}_i)-\nabla u(\mathbf{x}_i)\big\|^2
\le
C_3\,h^2\,\|u\|_{C^2(\Omega)}^2
\label{eq:grad_recovery_error}
\end{equation}

Finally, by the triangle inequality and $(a+b)^2\le 2a^2+2b^2$, combining \eqref{eq:grad_diff_bound_correct} and \eqref{eq:grad_recovery_error} yields \eqref{eq:h1_bound_local}.
Averaging \eqref{eq:h1_bound_local} over $i$ gives \eqref{eq:h1_bound_global}, with $\mathcal{L}_{\mathrm{MNP}}$ and $\mathcal{L}_{\mathrm{node}}$ as defined above.
\end{proof}

\autoref{thm:sobolev} formalizes MNP as a Sobolev-type regularizer: patch reconstruction accuracy controls a discrete gradient (hence flux) error.
Importantly, because the WLS gradient uses differences $(v_j-v_i)$, the gradient bound \emph{must} depend on both neighbor errors ($\mathcal{L}_{\mathrm{MNP}}$) and the center error ($\mathcal{L}_{\mathrm{node}}$).
In our training objective $\mathcal{L}=\mathcal{L}_{\mathrm{main}}+\alpha\,\mathcal{L}_{\mathrm{MNP}}$, the center term is exactly $\mathcal{L}_{\mathrm{main}}$ (up to the choice of $\ell$), so controlling the total loss controls an averaged discrete $H^1$-type error.
If one additionally supervises the center token inside the MNP star, the same analysis holds with $\mathcal{L}_{\mathrm{node}}$ absorbed into the patch loss.

\subsection{Temporal Correction as Predictor-Corrector Integration}
\label{subsec:theory_temporal}

A common neural rollout uses a ResNet-style one-step update $\mathbf{Z}^{t+1}=\mathbf{Z}^t+\Phi(\mathbf{Z}^t)$.
When $\Phi(\mathbf{Z}^t)\approx \Delta t\,F(\mathbf{Z}^t)$ for some latent vector field $F$, this is algebraically equivalent to \emph{Forward Euler} on the ODE $\dot{\mathbf{Z}}=F(\mathbf{Z})$.
Forward Euler is only conditionally stable; for stiff semi-discrete PDE systems this produces long-horizon drift.

We interpret our temporal block as a learned predictor-corrector:
\begin{equation}
\begin{aligned}
\tilde{\mathbf{Z}}^{t+1} &= \mathbf{Z}^t + \Phi^s(\mathbf{Z}^t)\\
\mathbf{Z}^{t+1} &= \mathbf{Z}^t + \Phi^t(\tilde{\mathbf{Z}}^{t+1}, \mathbf{Z}^t)
\end{aligned}
\label{eq:predictor_corrector_stage}
\end{equation}
where $\Phi^t$ is implemented by cross-attention between the predicted future $\tilde{\mathbf{Z}}^{t+1}$ (query) and the history/current state $\mathbf{Z}^t$ (keys/values), followed by a bounded gate.
This design gives the model enough information to approximate integrators that depend on both $t$ and $t{+}1$ states, i.e.\ implicit or semi-implicit schemes.

\begin{theorem}[\textbf{Stability region expansion via an embedded $\theta$-method}]
\label{thm:theta_stability}
Consider the scalar linear test equation $y'(t)=\lambda y(t)$ with $\Re(\lambda)\le 0$.
For $\theta\in[0,1]$, the classical $\theta$-method is defined by the one-step update
\begin{equation}
y_{n+1}
=
y_n + \Delta t\big((1-\theta)\lambda y_n + \theta \lambda y_{n+1}\big)
\qquad\Longleftrightarrow\qquad
y_{n+1} = R_\theta(z)\,y_n,\;\; z:=\Delta t\,\lambda
\label{eq:theta_method_def}
\end{equation}
with amplification factor
\begin{equation}
R_\theta(z) = \frac{1+(1-\theta)z}{1-\theta z}
\label{eq:theta_amplification}
\end{equation}
Then:
\begin{enumerate}
\item (\emph{A-stability}) If $\theta\ge \tfrac12$, then $|R_\theta(z)|\le 1$ for all $z\in\mathbb{C}$ with $\Re(z)\le 0$.
\item (\emph{Strict expansion over Forward Euler}) Forward Euler is the special case $\theta=0$ with $R_0(z)=1+z$, which is stable only for $z$ in the disk $|1+z|\le 1$.
In contrast, any $\theta\ge \tfrac12$ is stable for \emph{all} $\Re(z)\le 0$, i.e.\ removes the CFL-type restriction for the test problem.
\end{enumerate}
Moreover, the two-stage update \eqref{eq:predictor_corrector_stage} can represent the family \eqref{eq:theta_method_def} by learning an effective $\theta$ through the bounded gate and by using the future-conditioned correction to approximate the implicit dependence on $y_{n+1}$.
Hence, the temporal correction block enlarges the class of stable one-step maps that the network can realize compared to an explicit residual update.
\end{theorem}

\begin{proof}
The equivalence in \eqref{eq:theta_method_def} is obtained by collecting the $y_{n+1}$ terms:
\[
(1-\theta z)\,y_{n+1} = (1+(1-\theta)z)\,y_n
\]
which yields \eqref{eq:theta_amplification} provided $1-\theta z\neq 0$.

We now prove A-stability for $\theta\ge \tfrac12$.
Let $z=a+ib$ with $a=\Re(z)\le 0$.
Then
\[
|R_\theta(z)|^2
=
\frac{|1+(1-\theta)z|^2}{|1-\theta z|^2}
=
\frac{(1+(1-\theta)a)^2 + (1-\theta)^2 b^2}{(1-\theta a)^2 + \theta^2 b^2}
\]
Thus, $|R_\theta(z)|\le 1$ is equivalent to showing the denominator is at least the numerator:
\begin{align*}
&(1-\theta a)^2 + \theta^2 b^2
-
\Big((1+(1-\theta)a)^2 + (1-\theta)^2 b^2\Big)\\
&=
\Big[(1-\theta a)^2 - (1+(1-\theta)a)^2\Big]
+
\Big[\theta^2-(1-\theta)^2\Big]b^2
\end{align*}
Expanding the squares,
\[
(1-\theta a)^2 - (1+(1-\theta)a)^2
=
(1-2\theta a+\theta^2 a^2) - (1+2(1-\theta)a+(1-\theta)^2 a^2)
=
-2a + (2\theta-1)a^2
\]
and also
\[
\theta^2-(1-\theta)^2 = \theta^2 - (1-2\theta+\theta^2) = 2\theta-1
\]
Therefore,
\begin{equation}
(1-\theta a)^2 + \theta^2 b^2
-
\Big((1+(1-\theta)a)^2 + (1-\theta)^2 b^2\Big)
=
-2a + (2\theta-1)(a^2+b^2)
\label{eq:theta_difference}
\end{equation}
If $a\le 0$, then $-2a\ge 0$; if $\theta\ge \tfrac12$, then $2\theta-1\ge 0$, and $a^2+b^2\ge 0$.
Hence the right-hand side of \eqref{eq:theta_difference} is nonnegative, proving the denominator dominates the numerator and thus $|R_\theta(z)|\le 1$.
This establishes A-stability for $\theta\ge \tfrac12$.

Finally, for $\theta=0$ we have $R_0(z)=1+z$, so stability requires $|1+z|\le 1$, i.e.\ the classical Forward Euler stability disk, which excludes large negative real $z$ and therefore imposes a step-size restriction for stiff modes.
This proves the strict stability region expansion claimed in the theorem.
\end{proof}

\paragraph{Interpretation.}
Theorem~\ref{thm:theta_stability} formalizes the numerical role of the temporal correction block:
a two-stage update that can emulate a $\theta$-method (e.g.\ $\theta=\tfrac12$ corresponding to Crank-Nicolson / trapezoidal rule) has a fundamentally larger stability region than Forward Euler.
In our architecture, the cross-attention (future-conditioned correction) supplies the information pathway needed to approximate the implicit dependence on the next state, while the bounded gate provides a mechanism to adapt the effective $\theta$ locally, damping stiff/high-frequency error modes that otherwise destabilize long rollouts.

\section{Additional experiments}
\label{sec:add-exp}



\begin{figure}[ht!]
  \centering
  \begin{minipage}[t]{0.48\textwidth}\vspace{0pt}
    \centering
    \includegraphics[width=\linewidth]{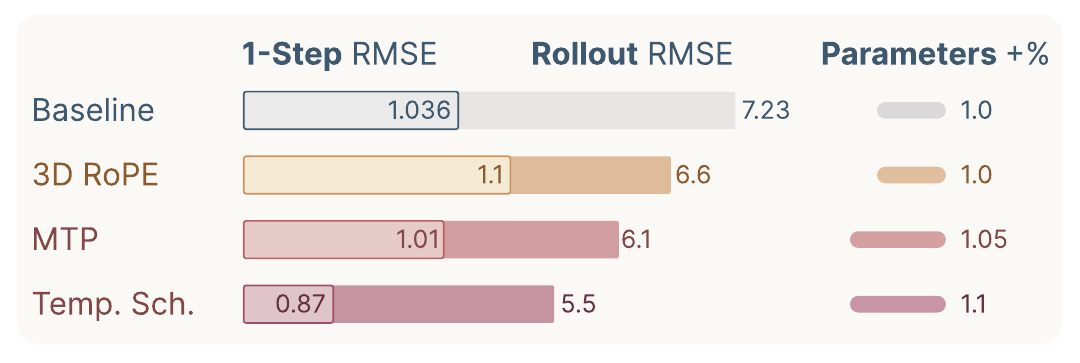}
    \captionof{figure}{\textbf{Impact of each upgrade.}
    We detail the improvement in term of 1-step and all-rollout RMSE on the Aneurysm dataset for the Transformer architecture. We also detail the increase in terms of trainable parameters.}
    \label{fig:improv}
  \end{minipage}\hfill
  \begin{minipage}[t]{0.48\textwidth}\vspace{0pt}
    \raggedright
    \paragraph{Width vs. Depth}
    For a given budget (in terms of trainable parameters), we study the optimal number of layers to achieve the lowest all-rollout RMSE. We display the results of a Transolver architecture on the Cylinder Dataset. Overall, we find a clear bottleneck around 15 layers, no matter the increase in width. We find a similar number of optimal layers in other models, and in other datasets. We believe this is related both to the solver used to generate datasets, and to the graph based architecture that struggles with vanishing gradient after too many spatial processing \cite{digiovanni2023oversquashing}. In lights of these results, we never go above 15 layers in any models we trained in this paper. Results are presented in \autoref{fig:width}.
  \end{minipage}
\end{figure}

\begin{figure*}[tbh!]
  \centering
  \includegraphics[width=0.99\textwidth]{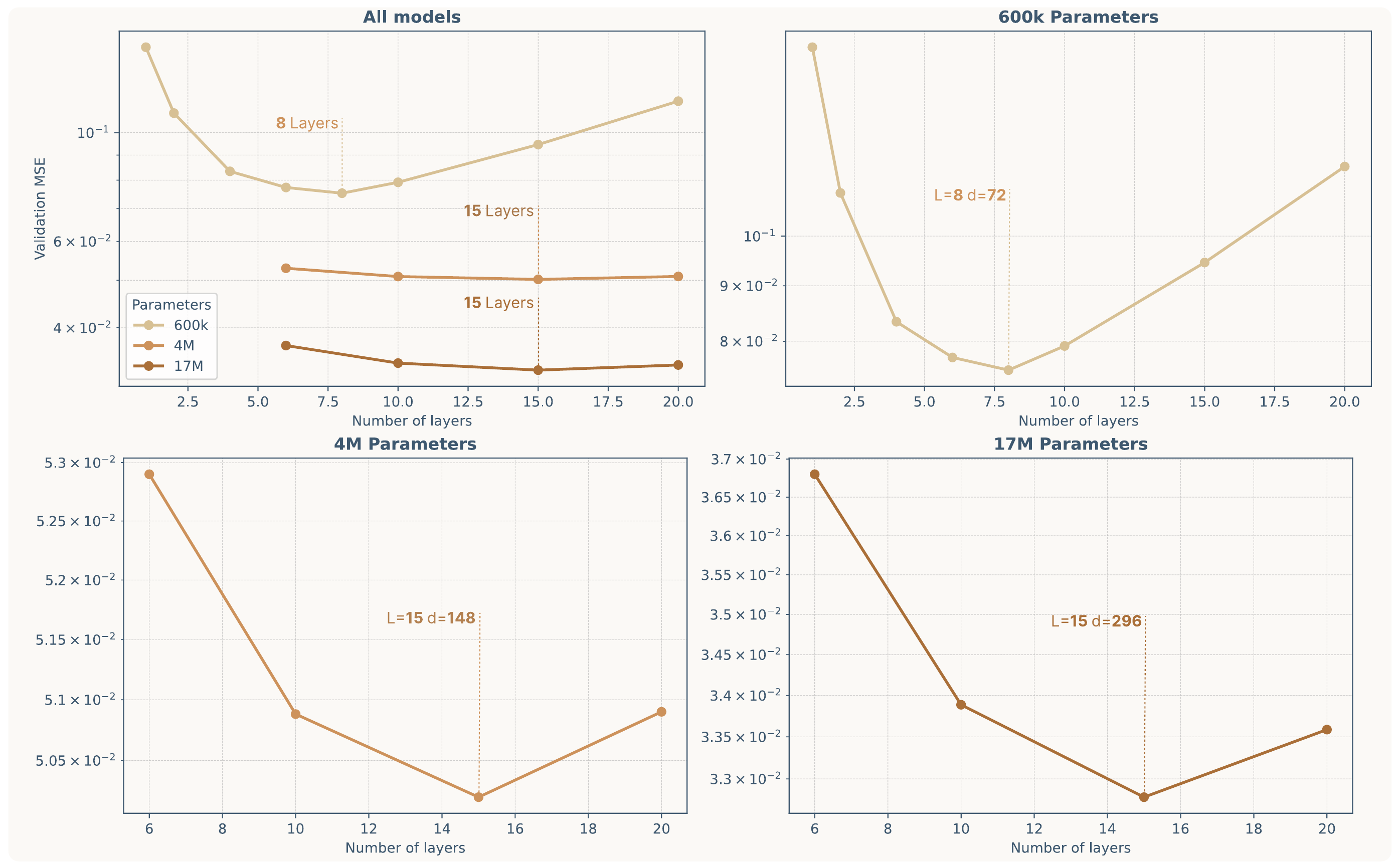}
  \caption{We study the performance of a transolver model under a fixed number of parameters. We train models with 600k, 4M and 17M parameters, and very accordingly the number of layers and the hidden dimension to keep the parameters constant. }
  \label{fig:width}
\end{figure*}

\paragraph{Additional Ablations} We present several additional experiments that were made in \autoref{fig:other-ablation}. First, we investigate attention variations:
\begin{itemize}
    \item Multi Head Latent Attention (MHLA) \cite{liu2024deepseek}
    \item Gated Delta Net in a 1 to 1 ratio \cite{kimiteam2025kimilinearexpressiveefficient}
    \item Gated Attention \cite{kimiteam2025kimilinearexpressiveefficient}
\end{itemize}

While MHLA does provide slightly better performance, it comes at a cost in terms of training time. Overall, we believe the tradeoff is not necessarily beneficial. We obtain similar findings to the two other approaches. 

Finally, we also investigate with two additional losses. \hypertarget{upd:physics-loss}{}The first one \crblue{consists of} adding the divergence of the predicted velocity field, 
(which is supposed to be zero since the fluids are incompressible). \crblue{This divergence residual improves 1-step RMSE, but it overfits the training distribution and increases all-rollout RMSE, so we do not retain it in the final framework. This observation is empirical for our setup and does not imply that physics-informed losses are generally ineffective.} We also added a second supervised loss: the cosine similarity between the target and predicted fields. The goal of such a loss is to \crblue{enforce} the predicted fields to also \crblue{align} in terms of orientations with the target. This offers a small improvement in long term prediction. 

\begin{figure*}[tbh!]
  \centering
  \includegraphics[width=0.99\textwidth]{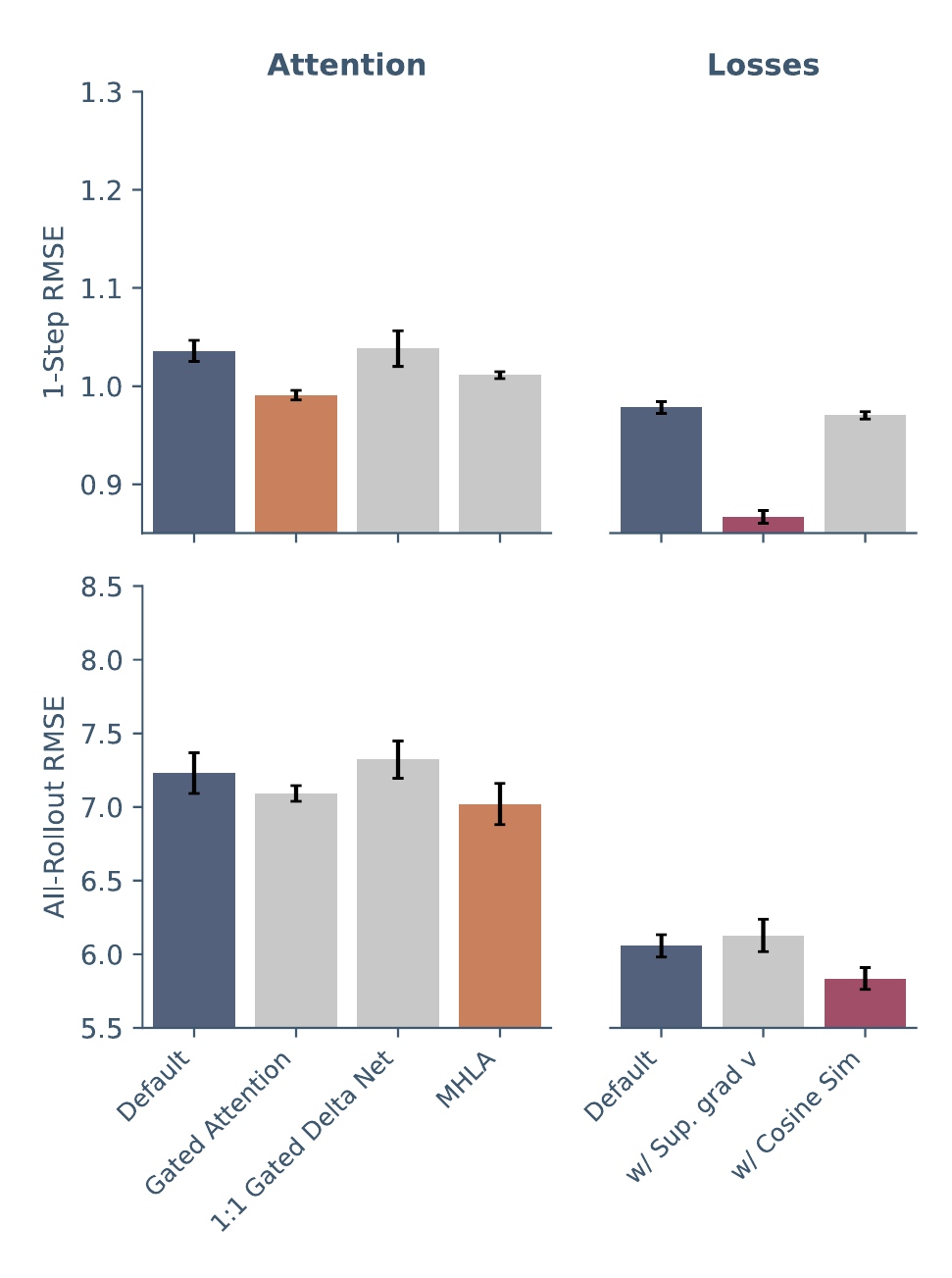}
  \caption{\textbf{Extra ablation studies.} We present the impact of several variants of attention and of different loss functions.}
  \label{fig:other-ablation}
\end{figure*}

\paragraph{Latent Representation}

We also extract the latent representation after half the spatial processing with and without Multi Node Prediction, compute a PCA of said representation, and then plot the three main components as RGB pixels. Results are presented in \autoref{fig:latentpca}. We can see that the latent representation with MNP is much crisper and interesting that the one without MNP. More interestingly, the latent representation closely resemble the pressure field at this time, while it was never shown in any inputs to the model. The latent representation at the first and final layer does not exhibit such phenomenon, being much closely related to the velocity field itself. This suggest that models may actually be learning adjacent physics on top of the task they are supervised on.

\begin{figure*}[tbh!]
  \centering
  \includegraphics[width=0.99\textwidth]{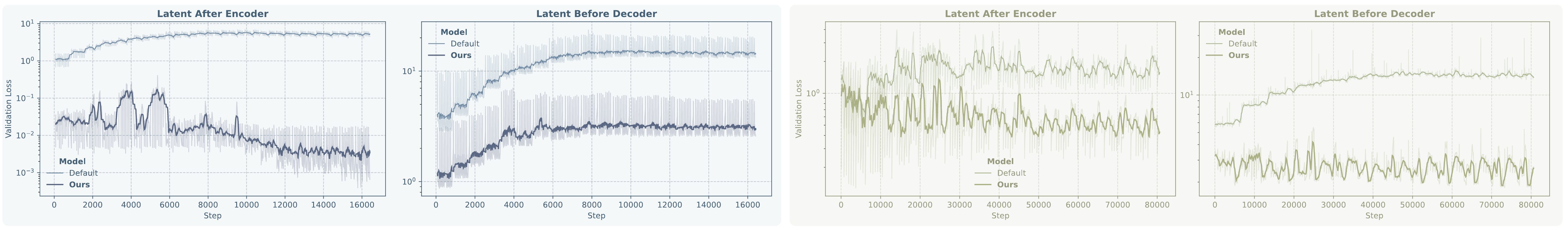}
  \caption{\textbf{Effect on latent representation.} While training architecture with and without Multi Node Prediction, we encode the next-step target to obtain a latent target, and compute it's difference with the current latent representation at each stage of the architecture. We display the differences after zero spatial processing, and after $L$ spatial processing.}
  \label{fig:latent}
\end{figure*}

\begin{figure*}[tbh!]
  \centering
  \includegraphics[width=0.99\textwidth]{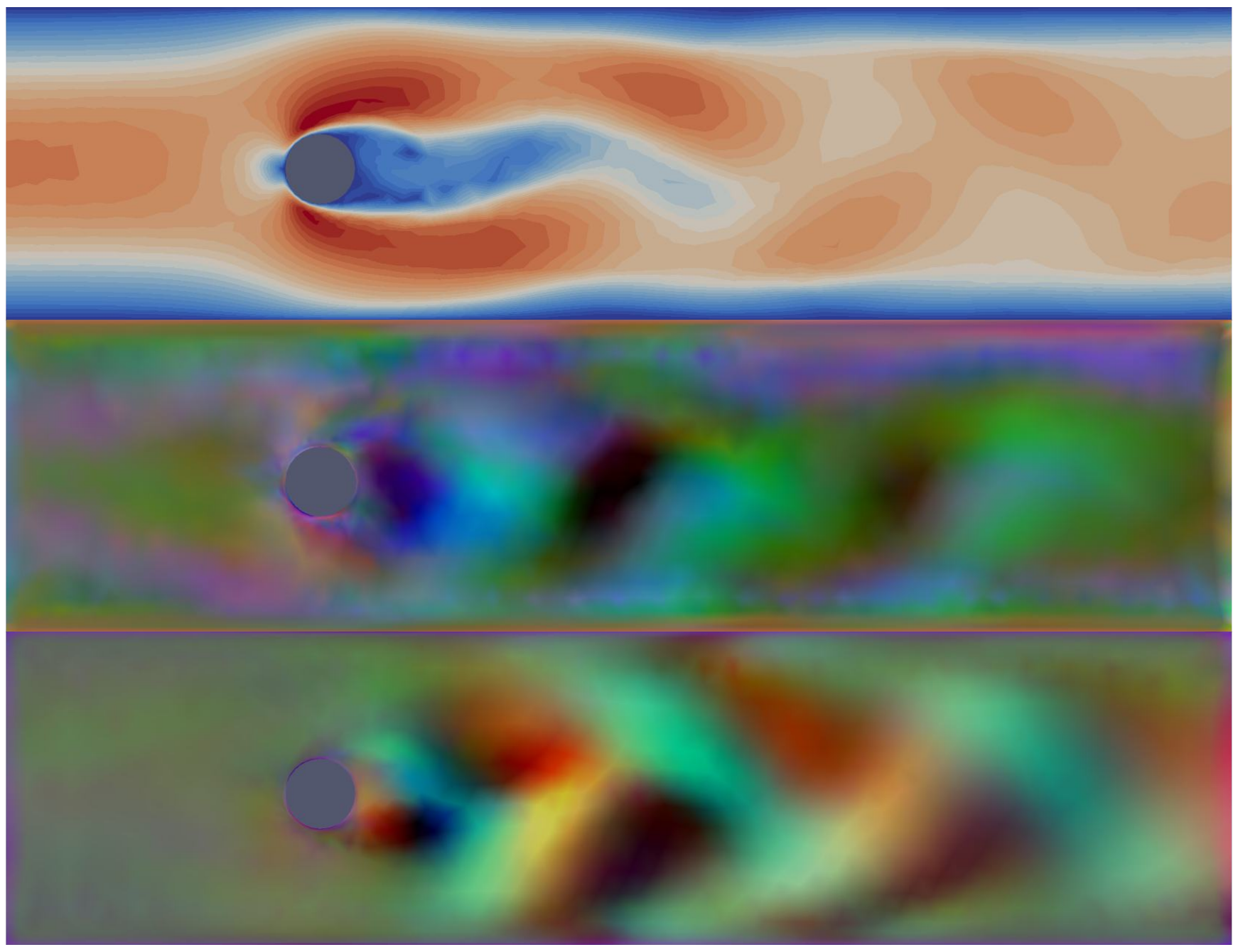}
  \caption{The next step velocity (target field) is presented in the top row. The latent representation without MNP is presented in the second row, while the latent representation with MNP is presented in the bottom row.}
  \label{fig:latentpca}
\end{figure*}

\end{document}